\documentclass[]{fairmeta}

\usepackage[utf8]{inputenc} 
\usepackage[T1]{fontenc}    
\usepackage{lmodern}       
\usepackage[english]{babel}
\makeatletter
\renewcommand{\title}[1]{%
  \gdef\titlelist{{\fontsize{19}{23}\selectfont\sffamily\bfseries #1}}%
}\makeatother

\usepackage{xspace}
\usepackage{wrapfig}
\usepackage{algorithm}        
\usepackage[noend]{algpseudocode}  
\algrenewcommand\algorithmicrequire{\textbf{Inputs:}}
\algrenewcommand\algorithmicensure{\textbf{Outputs:}}

\usepackage{amssymb}
\usepackage{pifont}
\usepackage{tcolorbox}

\usepackage{amsmath,amsfonts,bm}

\def\eqref#1{equation~\ref{#1}}

\def\1{\bm{1}}

\DeclareMathAlphabet{\mathsfit}{\encodingdefault}{\sfdefault}{m}{sl}
\SetMathAlphabet{\mathsfit}{bold}{\encodingdefault}{\sfdefault}{bx}{n}

\newcommand{\E}{\mathbb{E}}

\newcommand{\R}{\mathbb{R}}

\usepackage{wrapfig}

\usepackage{hyperref}       
\usepackage{url}

\usepackage{booktabs}       
\usepackage{amsfonts}       
\usepackage{nicefrac}       
\usepackage{microtype}      
\usepackage{xcolor}         

\definecolor{verylightgray}{gray}{0.93}

\usepackage{graphicx}
\usepackage{subcaption} 
\usepackage{wrapfig}

\usepackage{amsmath,amssymb,mathtools}
\usepackage{amsthm}

\usepackage{algorithm}
\usepackage{algpseudocode}

\usepackage{array}
\usepackage{multirow}
\usepackage{tabularx}
\usepackage{threeparttable}
\usepackage{colortbl}
\usepackage{siunitx}
\usepackage{tcolorbox}

\usepackage{newfloat}
\usepackage{listings}
\DeclareCaptionStyle{ruled}{labelfont=normalfont,labelsep=colon,strut=off}
\floatstyle{ruled}
\newfloat{listing}{tb}{lst}{}
\floatname{listing}{Listing}
\usepackage{enumitem}
\usepackage{cancel}
\usepackage{xspace}
\usepackage{pifont}
\usepackage{soul} 

\usepackage{booktabs,multirow,tabularx,array}
\newcommand{\ourmethod}{RRD\xspace}
\newtheorem{theorem}{Theorem}
\newtheorem{lemma}{Lemma}

\theoremstyle{definition}

\renewcommand{\epsilon}{\varepsilon}
\DeclareMathOperator{\op}{op}
\theoremstyle{plain}
\theoremstyle{remark}

\usepackage{bm} 
\newcommand{\vect}[1]{\bm{#1}} 

\newcommand{\norm}[1]{\left\lVert #1 \right\rVert}
\newcommand{\ip}[2]{\left\langle #1,\,#2 \right\rangle}

\newcommand{\tr}{\operatorname{tr}}

\usepackage[most]{tcolorbox}
\usepackage{enumitem}
\usepackage{listings}
\newtcolorbox{promptenv}[2][]{%
  enhanced,
  breakable,
  colback=white,
  colframe=black!,
  boxrule=0.5pt,
  arc=2mm,
  left=2mm,right=2mm,top=2mm,bottom=2mm,
  title={#2}, 
  #1
}

\title{Rethinking Rubric Generation for Improving LLM Judge and Reward Modeling for Open-ended Tasks} 

\author[\star]{William F. Shen}
\author{Xinchi Qiu}
\author{Chenxi Whitehouse}
\author[\star]{Lisa Alazraki}
\author[\star]{Shashwat Goel}
\author{Francesco Barbieri}
\author{Timon Willi}
\author{Akhil Mathur}
\author{Ilias Leontiadis}

\affiliation[]{Meta Superintelligence Labs}

\contribution[\star]{Work done at Meta}

\abstract{
Recently, rubrics have been used to guide LLM judges in capturing subjective, nuanced, multi-dimensional human preferences, and have been extended from evaluation to reward signals for reinforcement fine-tuning (RFT).
However, rubric generation remain hard to control: rubrics often lack coverage, conflate dimensions, misalign preference direction, and contain redundant or highly correlated criteria - degrading judge accuracy and producing suboptimal rewards during RFT.
We propose \ourmethod, a principled framework for rubric refinement built on a recursive decompose–filter cycle. \ourmethod decomposes coarse rubrics into fine-grained, discriminative criteria, expanding coverage while sharpening separation between responses. A complementary filtering mechanism removes misaligned and redundant rubrics, and a correlation-aware weighting scheme to prevent over-representing highly correlated criteria, yielding rubric sets that are informative, comprehensive, and non-redundant. 
Empirically, \ourmethod delivers large, consistent gains across both evaluation and training: it improves preference-judgment accuracy on JudgeBench and PPE for both GPT-4o and Llama3.1-405B judges, achieving top performance in all settings with up to $+17.7$ points on JudgeBench. When used as the reward source for RFT on WildChat, it yields substantially stronger and more stable learning signals, boosting reward by up to $160\%$ (Qwen3-4B) and $60\%$ (Llama3.1-8B) versus $\sim$10--20\% for prior rubric baselines, with gains that transfer to HealthBench-Hard and BiGGen Bench. Overall, \ourmethod establish recursive rubric refinement as a scalable and interpretable foundation for LLM judging and reward modeling in open-ended domains.
}

\date{\today}
\correspondence{\email{fs604@cam.ac.uk}, \email{iliasl@meta.com}}

\begin{document}

\maketitle

\section{Introduction}\label{sec:intro}

Large language models (LLMs) are increasingly used as judges (``LLM judge'') to evaluate open-ended generations (e.g., creative writing, planning, and roleplay)~\citep{gu2024survey}. However, because quality in these settings is subjective and inherently multi-attribute, LLM-based judging remains brittle: even frontier models can be near chance on preference benchmarks, and real-world use is further limited by bias, inconsistency, and limited transparency, especially when evaluation criteria are implicit or underspecified~\citep{thakur2024judging,tan2024judgebench,pezeshkpour2023large,saito2023verbosity,zheng2023judging,haldar2025rating,kim2023prometheus,gajcin2025interpreting}.

In parallel, recent work extends Reinforcement Learning from Verifiable Rewards (RLVR) beyond domains with objectively checkable outcomes (e.g., math and coding) to open-ended tasks under non-verifiable rewards~\citep{cui2025process,guo2025deepseek,lambert2024tulu}. A common approach converts binary correctness into pairwise preferences~\citep{ivison2024unpacking}, but this approach not only inherits the limitations of standard LLM-judge labeling, but also introduces a key scaling bottleneck: generating preference supervision that incorporates diverse and robust criteria matching the complexity of real-world reasoning~\citep{ouyang2022training,chen2024odin,singhal2023long,wang2024arithmetic,ye2025improving}.

One promising direction is rubric-based judging, where LLMs generate explicit, structured rubrics and use them to ground assessments, improving reliability and interpretability relative to holistic judgments~\citep{hashemi2024llm}. These rubric-level signals also integrate naturally into reinforcement fine-tuning (RFT) as structured rewards, motivating the Rubrics-as-Rewards paradigm for aligning models on complex open-ended tasks~\citep{gunjal2025rubrics}.

\begin{wrapfigure}{l}{0.52\textwidth}
  \centering
  \includegraphics[width=0.50\textwidth]{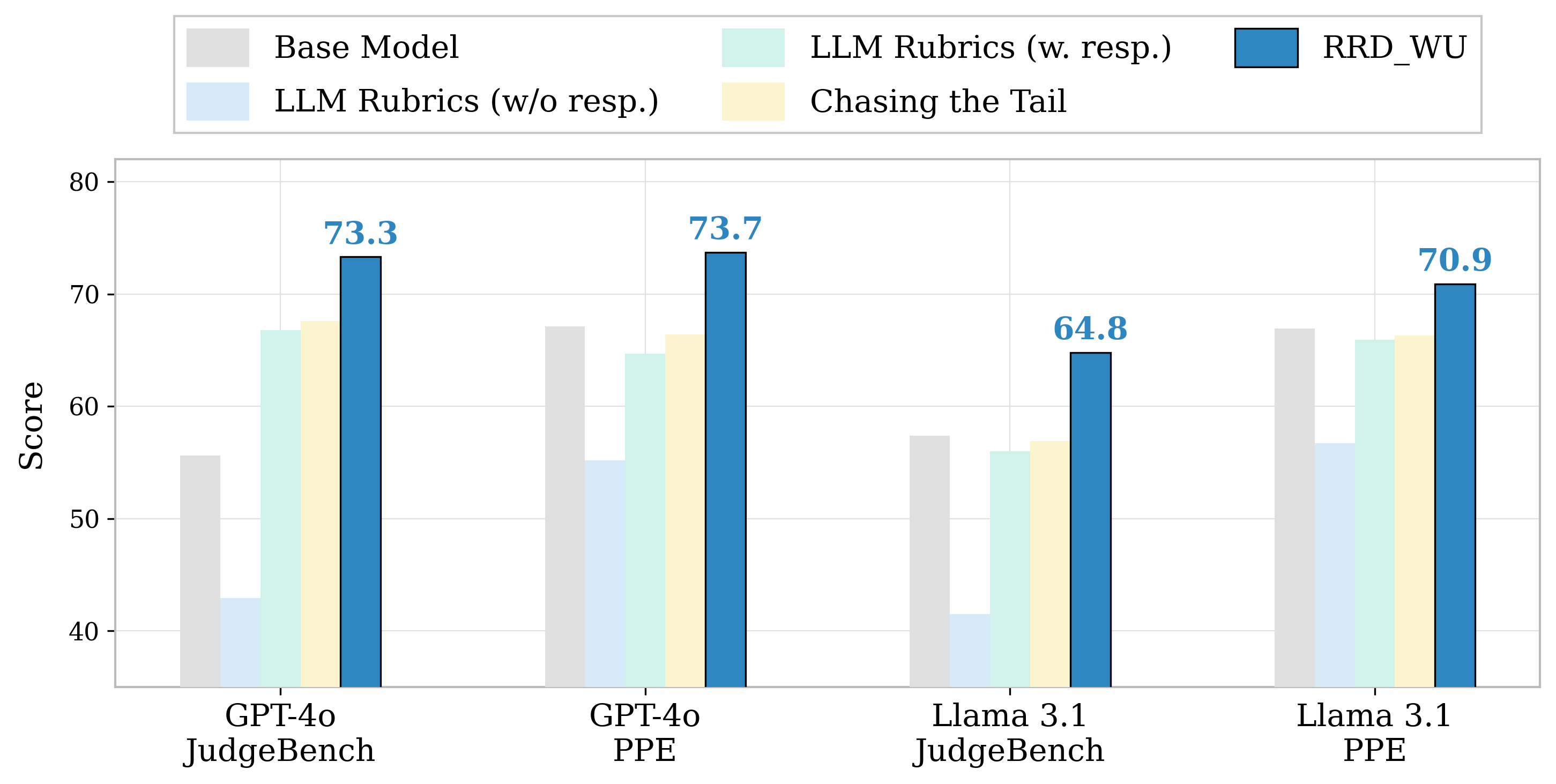}
  \caption{
  \ourmethod consistently outperforms all baselines on both JudgeBench and PPE for both proprietary (GPT-4o) and open-weights (Llama3.1-405B) judges, delivering substantial gains in preference-judgment accuracy.
  }
  \label{fig:judge_plot_wrap}
\end{wrapfigure}

However, existing rubric generation methods face two key limitations: (1) \textit{coverage deficiency}, where the rubric set fails to comprehensively capture the diverse and nuanced dimensions of generation quality; (2) \textit{noisy evaluation outcomes}, where misaligned, overlapping, or highly correlated rubrics introduce unreliable signals. Consequently, rubric-based judges often show weak alignment with human preferences, reducing evaluation accuracy, and when used as rewards for RFT, they yield only limited improvements in learning preference-aligned behavior across diverse model outputs. These flaws are severe: we show that naively generated rubrics degrade GPT-4o's judgment accuracy from 55.6\% to 42.9\% (Figure \ref{fig:judge_plot_wrap}) on JudgeBench \citep{tan2024judgebench} -- 13 points below using no rubrics at all.

In this work, we propose Recursive Rubric Decomposition (\ourmethod), a principled framework that improves both the accuracy of rubric-based judges and their effectiveness as reward models for downstream RFT. \ourmethod expands coverage by decomposing high-level rubric items—broad criteria satisfied by many responses—into finer-grained, nuanced subpoints. This yields more comprehensive and discriminative evaluations, helping LLM judges better capture subtle yet consequential quality differences when comparing candidate responses.
To complement rubric expansion, we introduce an aggregation mechanism to mitigate noise by (i) filtering rubrics that produce misaligned or conflicting signals, and (ii) removing redundant rubrics while down-weighting highly correlated ones. The first step guards against criteria that yield grossly incorrect judgments, while the second prevents overlapping perspectives from being overrepresented. Together, these components produce more stable and informative rubric-based assessments.
This process mirrors many real-world evaluations which follow a structured assessment rather than a single holistic verdict. Consider a physician diagnosing ambiguous symptoms: they don't render a holistic verdict, but forms hypotheses and orders discriminating tests. When a test result is consistent with multiple conditions, she orders more specific tests that distinguish between them. For example, a positive result for ``inflammatory markers'' doesn't determine whether the cause is autoimmune, infectious, or malignant - she must decompose further. The process terminates when tests discriminate: the remaining diagnosis is the only one consistent with all evidence. Crucially, she also recognizes correlated indicators—elevated CRP and ESR both signal inflammation but shouldn't be double-counted as independent evidence. \ourmethod instantiates this diagnostic logic: rubrics satisfied by multiple responses are insufficiently discriminative and are recursively decomposed; the process adapts naturally to case complexity, and correlated criteria are appropriately down-weighted.

We empirical demonstrate the performance of \ourmethod two-fold. We first evaluate \ourmethod on two widely used preference-judgment benchmarks: JudgeBench \citep{tan2024judgebench} and Preference Proxy Evaluation (PPE) \citep{frick2024evaluate}. Across both GPT-4o and Llama3.1-405B judges, \ourmethod consistently yields stronger agreement with human pairwise preferences, improving accuracy over the base judges and prior rubric-based baselines. In particular, our best variant (\ourmethod$_\text{WU}$) achieves the top score on all settings achieving up to 17.7\% improvement on JudgeBench for GPT-4o.

In addition, we demonstrate the effectiveness of \ourmethod in RFT by training open-source policies, Qwen3-4B~\citep{yang2025qwen3} and Llama3.1-8B~\citep{dubey2024llama}, with rubric-based judges on WildChat \citep{zhao2024wildchat} as the reward signal. Compared to prior LLM-judge-based and rubrics-based reward baselines, \ourmethod yields substantially stronger and better-calibrated rewards, translating into faster learning and higher final reward (Figure~\ref{fig:training_reward_curve}). Concretely, our method boosts reward by up to 160\% for Qwen3-4B and 60\% for Llama3.1-8B, versus only $\sim$10--20\% for LLM Rubrics and Chasing the Tail (Figure~\ref{fig:training_reward_curve}). These gains also carry over to downstream evaluations, where the resulting policies consistently improve on BiGGen Bench \citep{kim2024biggenbenchprincipledbenchmark} and HealthBench-Hard \citep{arora2025healthbench} for both model families.


Overall, our results show that \emph{recursive rubric decomposition} (\ourmethod) is a key enabler of effective rubric-based judging and reward modeling for open-ended tasks, turning brittle rubrics into reliable signals that support robust alignment in open-ended language generation.


\section{\ourmethod Framework}\label{sec:rrd_framework}

In this section, we first provide an overview of rubric-based judges in $\S$\ref{subsec:rubric_based_judge_intro}. We then formalize a theoretical perspective on rubric quality that motivates our approach. Finally, we present Recursive Rubric Decomposition (\ourmethod) in $\S$\ref{subsec:rrd_framework}, a principled framework that builds on this grounding to improve rubric-based evaluation.

\subsection{Rubric-based Judge Overview}\label{subsec:rubric_based_judge_intro}
Let $\mathcal{P}$ denote the prompt space and $\mathcal{R}$ the response space. For each $P\in\mathcal{P}$, consider a finite candidate set $R(P)=\{R_1,\dots,R_M\}\subseteq\mathcal{R}$. An LLM judge outputs a preference verdict $\mathcal{V}$ over $R(P)$.

A \emph{rubric-based} judge conditions its evaluation on a set of rubric predicates, evaluated separately. Each rubric is a measurable map
$
g:\mathcal{P}\times\mathcal{R}\to\{0,1\},
$
where $g(P,R)=1$ indicates that response $R$ satisfies the criterion under prompt $P$. Given a rubric family $\mathcal{G}=(g_1,\dots,g_m)$ and nonnegative weights $\vect{w}=(w_1,\dots,w_m)\in\mathbb{R}_+^m$, we define the rubric reward
\begin{equation}
f_{\vect{w},\mathcal{G}}(P,R)\;:=\;\sum_{k=1}^{m} w_k\, g_k(P,R)\,.
\label{eq:agg_rubric_reward}
\end{equation}
The constraint $w_k\ge 0$ is without loss of generality: any negatively polarized rubric can be rewritten as a positive, contrastive criterion. For two responses $R_i,R_j\in R(P)$, the judge prefers $R_i$ over $R_j$, denoted $R_i\succ R_j$, iff $f_{\vect{w},\mathcal{G}}(P,R_i) > f_{\vect{w},\mathcal{G}}(P,R_j)$.


\subsection{Theoretical Grounding for Rubric Quality}\label{subsec:theoretical_insights_rrd}

Since rubrics directly shape how LLM judges compare responses, and therefore the reward signals used in RFT, rubric quality is a primary lever for both judgment accuracy and reward fidelity. Ideally, a rubric set should be (a) \emph{informative}, with each criterion helping distinguish preferred from non-preferred responses; (b) \emph{comprehensive}, spanning the diverse dimensions along which quality varies; and (c) \emph{non-redundant}, providing complementary signals, since overlapping or correlated rubrics can distort aggregation. While prior work has made substantial progress (see \S\ref{sec:related}), achieving the three desiderata simultaneously remains challenging: expert-authored rubrics are often precise but can be limited in coverage, whereas LLM-generated rubrics scale readily but may include criteria that are overly generic, misaligned or highly correlated. More broadly, a principled theoretical account of rubric quality -- and how to systematically enforce it -- remains underdeveloped.

To bridge this gap and motivate \ourmethod, we first briefly present the theoretical grounding based on two assumptions: (A1) \emph{positive edge} and (A2) \emph{bounded correlation}. We then analyze the judge’s misclassification probability, whether aggregated rubric verdicts recover the true preference label, and show that it admits an exponential upper bound (Eq.~\ref{eq:upper_bound_prob_wrong_verdict}). Minimizing this bound yields a principled objective for rubric generation and correlation-aware weighting, directly operationalizing the three desiderata above.

Assume each rubric $g_k \in \mathcal{G}$ is weakly informative (a positive ``edge'' over random guessing) and its noise (the deviation of its verdicts from conditional mean given the true label) is sub-Gaussian with bounded dependence:
\begin{itemize}
\item[(A1)] (\emph{Positive edge}) There exist $\mu_k>0$ such that
$\mathbb{E}[\widehat Y_k\mid Y=+1]=+\mu_k,\ \mathbb{E}[\widehat Y_k\mid Y=-1]=-\mu_k$, where $Y \in \{\pm1\}$ is the ground-truth preference label, $\widehat Y_k$ is the verdict of the $k^{th}$ rubric. 

\item[(A2)] (\emph{Bounded correlation}) Letting $Z_k=\widehat Y_k-\mu_kY$, the vector $Z=(Z_1,\dots,Z_m)$ is mean-zero sub-Gaussian with covariance $\Sigma_y$ satisfying
$\mathrm{Var}(Z_k)\le \sigma_k^2$ and $|\mathrm{Corr}(Z_i,Z_j)|\le \rho<1$ for $i\neq j$.
\end{itemize}


In simple words, (A1) assumes that each rubric, with a positive edge ($\mu_k > 0$), contributes positively toward distinguishing the preferred response from the inferior one;
(A2) assumes that individual rubric's noise is bounded and rubrics are not repetitive (bounded pairwise correlation $\rho < 1$). Here, the noise $Z_k=\widehat Y_k-\mu_k Y$ captures rubric-specific randomness or systematic mismatch -- i.e., the part of a rubric's verdict not explained by the true label, arising from ambiguous criteria, imperfect judge execution, or instance-specific idiosyncrasies. This enables standard concentration bounds over the aggregated, non-redundant decisions.

The probability that the rubric-based judge, after aggregating decisions across all rubric items (Eq.\ref{eq:agg_rubric_reward}), produces an incorrect verdict (misclassification probability) $\widehat{Y} \ne Y$ is then upper bounded by:
\begin{equation}
\mathbb{P}(\widehat Y\ne Y)
\ \le\ 
\exp\!\Big(-\tfrac12\min\{\Delta_m^2/V_m(+1),\,\Delta_m^2/V_m(-1)\}\Big)
\label{eq:upper_bound_prob_wrong_verdict}
\end{equation}
where $\Delta_m:=\vect{w}^{\top}\vect{\mu}$ and $V_m:=\vect{w}^{\top}\bm{\Sigma}\,\vect{w}$ is the variance proxy of the weighted residuals (ref. Appendix~\ref{app:proof_upper_bound_prob_wrong_verdict}).


Eq.~\ref{eq:upper_bound_prob_wrong_verdict} implies that tightening the misclassification bound amounts to maximizing
$\Xi = \frac{(\vect{w}^\top \vect{\mu})^2}{\vect{w}^\top \vect{\Sigma}\vect{w}}$.
This perspective suggests a high-level prescription: 
\begin{enumerate}
    \item \textbf{Decomposition}: Decompose broad rubrics into finer dimensions to enhance coverage and discrimination.
    \item \textbf{Positivity}: Remove misaligned rubrics to eliminate negative edge and maintain constructivity.
    \item \textbf{Non-redundancy}: Prune redundant rubrics to ensure distinct, non-overlapping criteria.
    \item \textbf{Weight Optimization}: Prevent over-representation of highly correlated rubrics via weight optimization.
\end{enumerate}

In tandem, (1) and (2) expand the rubric set by admitting new criteria with positive edge ($\mu_k > 0$) to maximize aggregate edge, while (3) and (4) minimize redundancy and correlation within the denominator. Together, these yield an exponentially decaying misclassification probability and formalizes our three desiderata.
%


\subsection{Methodology: Recursive Rubric Decomposition (\ourmethod)}\label{subsec:rrd_framework} 

\input\begin{figure}[t]
  \centering
  \includegraphics[width=0.9\textwidth]{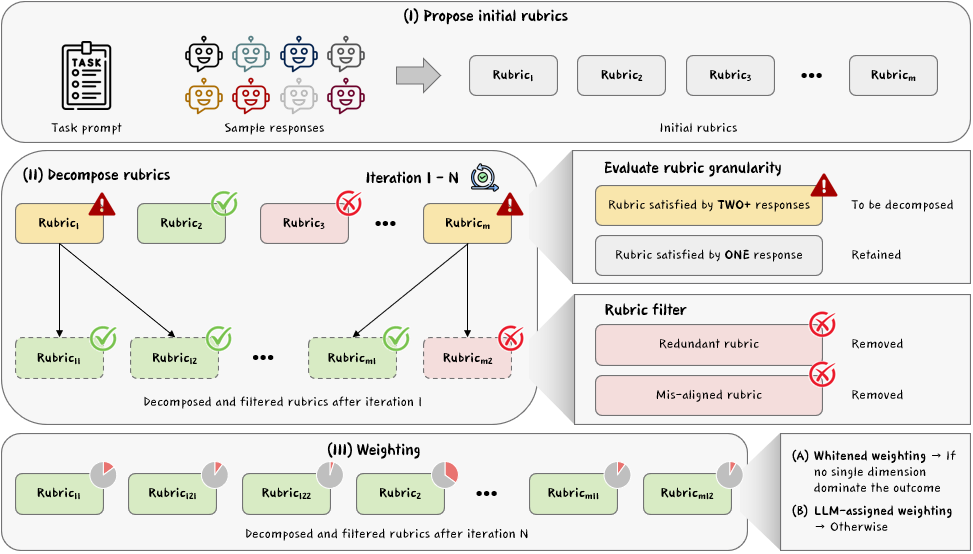}
  \caption{\textbf{Overview of \ourmethod framework}: \ourmethod consists of three stages: \textbf{(I) Initial Rubric Proposal.} LLM proposes initial candidate rubrics (conditioned on the task prompt and sample responses) for optimization. \textbf{(II) Recursive Decomposition and Filtering.} Recursively decompose coarse rubric into finer dimensions to enhance coverage and discrimination, while filtering misalgined and redundant rubrics. The cycle stops when the number of discarded rubrics exceeds $N$, indicating saturation in novel, non-redundant, and valid rubrics. \textbf{(III) Rubric Weight Assignment.} For open-ended tasks where preference signal is distributed, assign whitened uniform (WU) weights to account for correlation structure and prevent over-representation of highly correlated rubrics. Otherwise, assign LLM-proposed heuristic weights. Empirically, WU weighting yields higher LLM judge accuracy and improves the effectiveness of rubrics as generative rewards in RFT.}  
  \label{fig:my-image}
\end{figure}


Building on the theoretical insights in \S\ref{subsec:theoretical_insights_rrd}, we optimize Eq.\ref{eq:upper_bound_prob_wrong_verdict} through four pillars: \emph{decomposition}, \emph{positivity}, \emph{non-redundancy}, \emph{weight optimization}. Now, we introduce \emph{Recursive Rubric Decomposition} (\ourmethod), a principled rubric construction framework. 
The full \ourmethod procedure is summarized in Figure \ref{fig:my-image} and Algorithm \ref{alg:rrd}.


\ourmethod consists of three stages. First, we prompt an LLM to
propose initial rubrics conditioned on the task prompt and $m$ sample responses (we set $m=8$ in all experiments), yielding candidate rubrics for optimization.

Second, we start the recursive decomposition and filtering cycle. We exploit the fact that distinct responses must differ in some respects: a rubric satisfied by many responses is too broad-brush and insufficiently discriminative. It can be decomposed into finer sub-dimensions that capture more nuanced aspects of quality. We operationalize this by instructing an LLM-based rubric generator to recursively decompose any rubric that applies to more than $n$ rollouts (we use $n=2$ in all experiments, triggering decomposition whenever a rubric matches more than two candidate responses and leaving minimal room for under-decomposition). We then apply two filters: (a) misalignment filtering: which discards rubrics that prefer outputs from a weaker model (Llama3-8B) over a stronger model (GPT-4o) as a proxy for incorrect preference direction (more discussion about this in Appendix \ref{app:weak_strong_discussion}); and (b) LLM-based redundancy filtering, which removes rubrics that are substantially overlapping with existing ones.

We repeat this decomposition--filtering loop until the proposer struggles to produce novel, valid, non-redundant items. For efficiency, we use an early-stopping criterion: if the number of accumulated rejected proposals exceeds a termination threshold, we stop the loop, since further iterations are unlikely to yield effective new rubrics. We treat this termination threshold as a tunable hyperparameter. We use $15$ in all of our experiments below. The resulting rubric set is task-adaptive: its size emerges from the intrinsic complexity of the prompt rather than being fixed \emph{a priori} (More discussion about this in Appendix \ref{app:weak_strong_discussion}).

Finally, optimize weights to prevent over-representation of highly correlated rubrics. In practice, this is particularly challenging because the ground-truth preference labels needed to estimate the rubric edges $\vect{\mu}$ are unavailable. Prior work therefore instructs an LLM to assign rubric weights using a heuristically chosen scale \citep{gunjal2025rubrics, viswanathan2025checklists, zhang2025chasing}. This strategy is effective only if there exists a clearly \emph{dominant} rubric and the LLM can reliably identify it, so that ${(\vect{w}^\top \vect{\mu})^2}$ dominates the covariance term, effectively suppressing the impact of rubric correlations. In open-ended tasks, quality is inherently multi-dimensional, subjective, and nuanced, so the class-separating signal is often \emph{distributed} across many criteria rather than concentrated in single dominant rubric. This helps explain why existing rubric-based judges remain imperfect in practice, and why expanding coverage with additional valid rubrics further disperses the signal.

To sidestep the need for labeled edges, we instead minimize the misclassification upper bound by \emph{homogenizing signal in a whitened space}. This is feasible because the second-order redundancy structure of rubric scores, captured by $\vect{\Sigma}$, is an intrinsic property of the rubric set, not of any particular response pair, and can therefore be estimated from unlabeled data (ref. Appendix~\ref{app:weighting_in_whitened_space}). Concretely, we choose weights that \emph{whiten} the rubric space via $\Sigma^{-1/2}$ (i.e., removing correlations while applying equal weighting in the whitened coordinates). This yields a simple, label-free, and correlation-aware weighting scheme that stabilizes aggregation when signals are spread across many rubric dimensions.

\section{\ourmethod-based LLM Judge Results} \label{sec:judge_performance}
In \S\ref{sec:rrd_framework}, we introduced \ourmethod as a novel and theoretically grounded framework for rubric generation. In this section, we empirically evaluate the effectiveness of \ourmethod in improving the accuracy of LLM judges on open-form pairwise judgment tasks.

\paragraph{Dataset.} 
We evaluate the accuracy of LLM judge using 
(1) JudgeBench~\citep{tan2024judgebench}, which consists of challenging open-form response pairs spanning knowledge, reasoning, mathematics, and coding tasks. Following~\citet{tan2024judgebench}, we report results on the subset where both responses are generated by GPT-4o (350 preference pairs), and
(2) Preference Proxy Evaluations (PPE)~\citep{frick2024evaluate}, a large-scale benchmark containing 10.2K human preference pairs from Chatbot Arena covering 20 LLMs across 121+ languages. 

\paragraph{Baselines.}
We compare \ourmethod against the base model (preference labeling without explicit rubrics) and several rubric-based judge baselines that differ in their rubric generation strategies. Specifically, we include:
(1) LLM Rubrics (w/o resp.): rubrics are proposed based solely on the prompt;
(2) LLM Rubrics (w. resp.): rubrics are generated with access to sample responses for better grounding;
(3) Chasing the Tail~\citep{zhang2025chasing}: a state-of-the-art iterative method that optimizes rubrics to differentiate high-quality response pairs; and
(4) \ourmethod variants: we evaluate three weighting schemes -- $\ourmethod_\text{uniform}$ (uniform weights), $\ourmethod_\text{LLM}$ (LLM-assigned weights), and $\ourmethod_\text{UW}$ (whitened weights).

We employ GPT-4o~\citep{hurst2024gpt} and Llama-3.1-405B~\citep{dubey2024llama} as rubric generator and final judge, and consistently use GPT-4o and Gemini 2.5-pro~\citep{comanici2025gemini} as sample response generator (each model generates 4 sample responses). 

\begin{figure}[t]
  \centering

  \begin{subfigure}[t]{0.6\textwidth}
    \vspace{0pt}
    \centering
    
    {\setlength{\tabcolsep}{4pt}%
     \renewcommand{\arraystretch}{0.95}%
     \scalebox{0.85}{%
      \begin{tabular}{lcccc}
       
        \toprule
         &
        \multicolumn{2}{c}{\textit{GPT-4o}} &
        \multicolumn{2}{c}{\textit{Llama3.1-405B}} \\
         \cmidrule(lr){2-3}\cmidrule(lr){4-5}
        \textbf{Method} & \textbf{JudgeBench} & \textbf{PPE} & \textbf{JudgeBench} & \textbf{PPE} \\
        \midrule
        Base Model & 55.6 & 67.1 & 57.4 & 66.9 \\
        LLM Rubrics (w/o resp.) & 42.9 & 55.2 & 41.5 & 56.7 \\
        LLM Rubrics (w.\ resp.) & 66.8 & 64.7 & 56.0 & 65.9 \\
        Chasing the Tail & 67.6 & 66.4 & 56.9 & 66.3 \\
        $\ourmethod_\text{uniform}$ & 70.0 & 71.2 & 61.1 & 68.2 \\
        $\ourmethod_\text{LLM}$ & 72.4 & 72.6 & 63.5 & 69.7 \\
        \rowcolor{verylightgray} $\ourmethod_\text{WU}$ & \textbf{73.3} & \textbf{73.7} & \textbf{64.8} & \textbf{70.9} \\
        \bottomrule
      \end{tabular}%
     }
    }
    \subcaption*{(a)}
  \end{subfigure}\hfill
  \begin{subfigure}[t]{0.39\textwidth}
    \vspace{0pt}
    \centering
    \includegraphics[width=\linewidth]{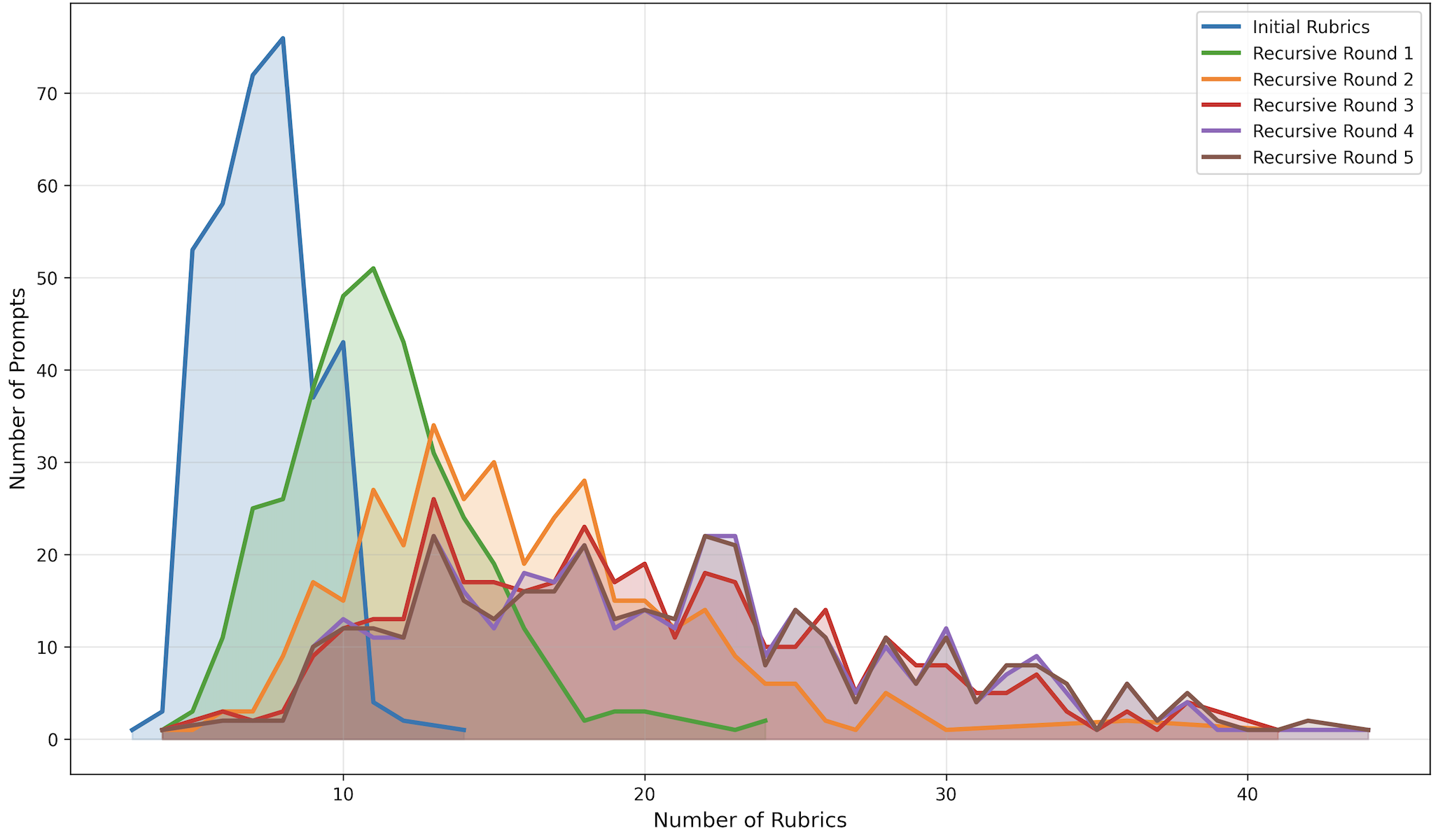}
    \subcaption*{(b)}
  \end{subfigure}

  \caption{
  (a) Accuracy on JudgeBench and PPE Preference datasets for base model and rubric-assisted judges under different rubric-generation strategies. While basic LLM-generated rubrics (unconditioned on sample responses) can degrade performance, \ourmethod yields consistent improvements over the baselines. Notably, $\ourmethod_\text{WU}$ delivers the largest gains and scales reliably across both proprietary (GPT-4o) and open-weights (Llama3.1-405B) judges.
  (b) Rubric-count dynamics on JudgeBench. Starting from an average of $7.4$ rubrics, the count rises to $\sim 20$, while the increasing variance across tasks indicates that the recursive procedure adapts evaluation depth to instance complexity.
  }
  \label{fig:table_and_plot}
\end{figure}
\subsection{Results}\label{subsec:rrd_based_judge_results}

We summarize the primary performance results in Figure~\ref{fig:table_and_plot}(a) and analyze the dynamics of the recursive generation process in Figure~\ref{fig:table_and_plot}(b). 

As shown in Figure~\ref{fig:table_and_plot}(a), $\ourmethod$ variants consistently outperform baselines across both models and datasets. On JudgeBench, $\ourmethod_\text{WU}$ improves GPT-4o's accuracy from $55.6\%$ to $73.3\%$ ($+17.7$ points). Similarly, Llama-3.1-405B sees a $6.6$ point boost. This suggests that the recursive discovery of latent rubrics provides critical signal that the initial rubric sets miss in a single pass. Interestingly, simple LLM Rubrics (w/o resp.) actually degrade performance compared to the base model as LLM judge. This highlights the failure mode where generic rubrics introduce noise or distract the judge. \ourmethod mitigates this by grounding the rubrics in recursive residuals, ensuring they capture meaningful differences. Additionally, whitened weighting ($\ourmethod_\text{WU}$) consistently outperforms the uniform and LLM-assigned weighting. This validates our theoretical framework: for open-ended tasks involving multiple non-dominating dimensions, accounting for rubric correlation yields a more robust aggregation of judge ``votes'' than simple averaging or LLM-based self-weighting.

Figure~\ref{fig:table_and_plot}(b) shows the evolution of rubric counts across recursive rounds on JudgeBench. Two key phenomena emerge: first, the initial average of $7.4$ rubrics grows rapidly and plateaus at approximately $20$ by iteration 3, suggesting that \ourmethod captures a comprehensive set of dimensions quickly before reaching a ``saturation'' point in novelty. Second, the framework is inherently adaptive, increasing the variance in rubric counts as it progresses. Consequently, recursion terminates quickly for simpler tasks but generates deeper criteria for complex tasks to resolve residual quality differences. A qualitative example of simpler tasks entailing fewer rubrics are provided in the Appendix~\ref{app:rubric_examples}.

\paragraph{Ablations.}
\begin{table}[t]
\centering

\begin{subtable}[t]{0.37\linewidth}
\centering
\scriptsize
\resizebox{\linewidth}{!}{%
    \begin{tabular}{l c c >{\columncolor{verylightgray}}c c}
        \toprule
        & \multicolumn{4}{c}{Termination Threshold} \\
        \cmidrule(lr){2-5}
        \textbf{Method} & \textbf{5} & \textbf{10} & \textbf{15} & \textbf{20} \\
        \midrule
        $\ourmethod$     & 68.6 & 67.4 & 70.0 & 66.3 \\
        $\ourmethod_\text{LLM}$ & 69.7 & 69.1 & 72.4 & 67.1 \\
        $\ourmethod_\text{WU}$  & 71.4 & 72.0 & 73.3 & 73.1 \\
        \bottomrule
    \end{tabular}%
}
\caption{}
\label{tab:side_a}
\end{subtable}
\hfill
\begin{subtable}[t]{0.60\linewidth}
\centering
\scriptsize
\resizebox{\linewidth}{!}{%
\begin{tabular}{l>{\columncolor{verylightgray}}c c c c}
    \toprule
    & \multicolumn{4}{c}{Sample Responses Generator} \\
    \cmidrule(lr){2-5}
    \textbf{Method} & \textbf{GPT+Gemini} & \textbf{Gemini Only} & \textbf{GPT Only} & \textbf{GPT+Llama-8B} \\
    \midrule
    \ourmethod     & 70.0 & 66.6 & 57.4 & 54.3 \\
    $\ourmethod_\text{LLM}$ & 72.4 & 71.0 & 61.4 & 60.0 \\
    $\ourmethod_\text{WU}$  & 73.3 & 72.4 & 62.3 & 60.3 \\
    \bottomrule
\end{tabular}}
\caption{}
\label{tab:side_b}
\end{subtable}

\caption{Ablation studies for \ourmethod on JudgeBench with GPT-4o as rubric proposer.
(a) Ablation of the termination threshold (i.e., the number of rubric proposals rejected due to redundancy or misalignment) for the \ourmethod process.
(b) Ablation of the sample response generation strategy, comparing (i) a combination of strong frontier models, (ii) a single strong model, and (iii) a mixture of strong and weaker models as inputs to the rubric proposer.}
\label{tab:judge_ablation}
\end{table}



To assess the impact of key design choices, we conduct ablation studies on JudgeBench, focusing on two key components: the termination threshold for recursion and the choice of sample response generator.


Table~\ref{tab:judge_ablation}(a) reports RRD variants under different termination thresholds, where the recursion stops after a fixed number of rubric proposals are rejected as redundant or misaligned. We observe that lower thresholds (e.g., $5$ or $10$) underperform, likely due to insufficient exploration of the rubric space. Raising the threshold to $20$ yields a plateau or slight drop for basic \ourmethod and $\ourmethod_\text{LLM}$, consistent with diminishing returns and increased noise from correlated criteria. In contrast, $\ourmethod_\text{WU}$ is notably robust: accuracy remains high even at threshold $20$, suggesting that whitening-uniform weighting effectively counteracts correlation-induced noise and stabilizes performance under deeper rubric exploration.


In Table~\ref{tab:judge_ablation}(b), we ablate the sample response generator by comparing multi-model ensembles with single-model and mixed-capability setups. Using two strong frontier models (GPT-4o + Gemini 2.5-Pro) performs best, consistently outperforming single-model baselines. Notably, when GPT-4o is the rubric proposer, it benefits more from conditioning on high-quality samples produced by a different frontier model (Gemini) than on its own outputs, suggesting that exposure to diverse reasoning styles and perspectives improves rubric generation. Meanwhile, replacing one frontier model with a smaller model (e.g., Llama3.1-8B) tends to reduce performance in \ourmethod setting, likely because decomposition benefits from high-quality samples that reveal when a rubric is overly coarse (i.e., satisfied by multiple strong responses) and should be refined.
\section{\ourmethod-based RFT} \label{sec:rft}

Previously in \S\ref{sec:judge_performance}, we showed that \ourmethod improves the accuracy of LLM judges on open-ended tasks by producing more comprehensive, discriminative and fine-grained rubrics. In this section, we ask whether these gains translate into learning gains when judges are used as generative reward models for reinforcement fine-tuning (RFT). RFT is the natural stress test for open-ended judging: even small systematic scoring biases can be amplified by the optimization loop, shaping model behavior in unintended ways. A reward model therefore must do more than standalone preference judges, it must provide stable, informative gradients that consistently promote the desired trade-offs across diverse prompts. We therefore use RFT to convert judge quality into an end-to-end signal, evaluating (i) reward reliability during training, and (ii) downstream performance of the resulting policies on in- and out-of-domain tasks.

\paragraph{Dataset and Training.} 
We conduct RFT on 4K English, non-toxic, de-duplicated prompts sampled from WildChat \citep{zhao2024wildchat}, representing natural user–AI interactions. We employ Dr.GRPO \citep{drgrpo} as our RFT algorithm for all settings (see details in Appendix~\ref{app:rft_training_details}).

Then after training, to demonstrate the performance of resulting models, we evaluate the resulting policies on two rubric-based, open-ended generation benchmarks: BiGGen Bench \citep{kim2024biggenbenchprincipledbenchmark} and HealthBench-Hard \citep{arora2025healthbench}. BiGGen Bench is a free-form generation benchmark spanning multiple core capabilities, and is evaluated with instance-specific criteria that capture what a good answer should contain for each prompt. HealthBench is a domain-specific benchmark containing 5K multi-turn clinical dialogues, scored using conversation-specific rubrics authored by physician experts, with importance-weighted criteria to reflect clinical priorities; the Hard split is designed to remain challenging and unsaturated.

We choose these two benchmarks because they provide complementary coverage of (i) in-distribution general assistant behavior and (ii) out-of-domain, high-stakes generalization. BiGGen Bench is closer to our training distribution (WildChat): both emphasize broad, real-world, open-ended assistant tasks, making BiGGen a natural test of whether RFT improves general helpfulness and capability across the kinds of prompts seen in the wild. In contrast, HealthBench-Hard focuses on clinical multi-turn dialogue with physician-defined notions of correctness, safety, and communication. It is substantially more domain-specific than open-domain chat, thereby stress-testing whether the learned policy maintain reliability in a high-stakes setting.

We report the \textit{macro-average score}, computed as the mean of per-example rubric scores
$s_i=\frac{\texttt{total\_awarded}_i}{\texttt{total\_possible}_i}$ (converted to a percentage), so that \textit{each dialogue/prompt contributes equally} regardless of how many rubric criteria it contains.
For HealthBench, we additionally \textit{clip the macro-average to $[0,100]$} after averaging, following the benchmark's recommended presentation to prevent rare negative-penalty cases from overly skewing the headline score.

\paragraph{Models}
We fine-tune Qwen3-4B \citep{yang2025qwen3} and Llama3.1-8B \citep{dubey2024llama} as policy models. GPT-4o serves as the rubric proposer and GPT-OSS-120B \citep{agarwal2025gpt} determines rubric satisfaction for rollout responses (or give direct preference judgment in ``LLM judge as reward'' baseline) during RFT.

\subsection{Results}\label{subsec:rrd_result}

\begin{figure}[t]
\centering
\begin{subfigure}{0.48\linewidth}
  \centering
  \includegraphics[width=\linewidth]{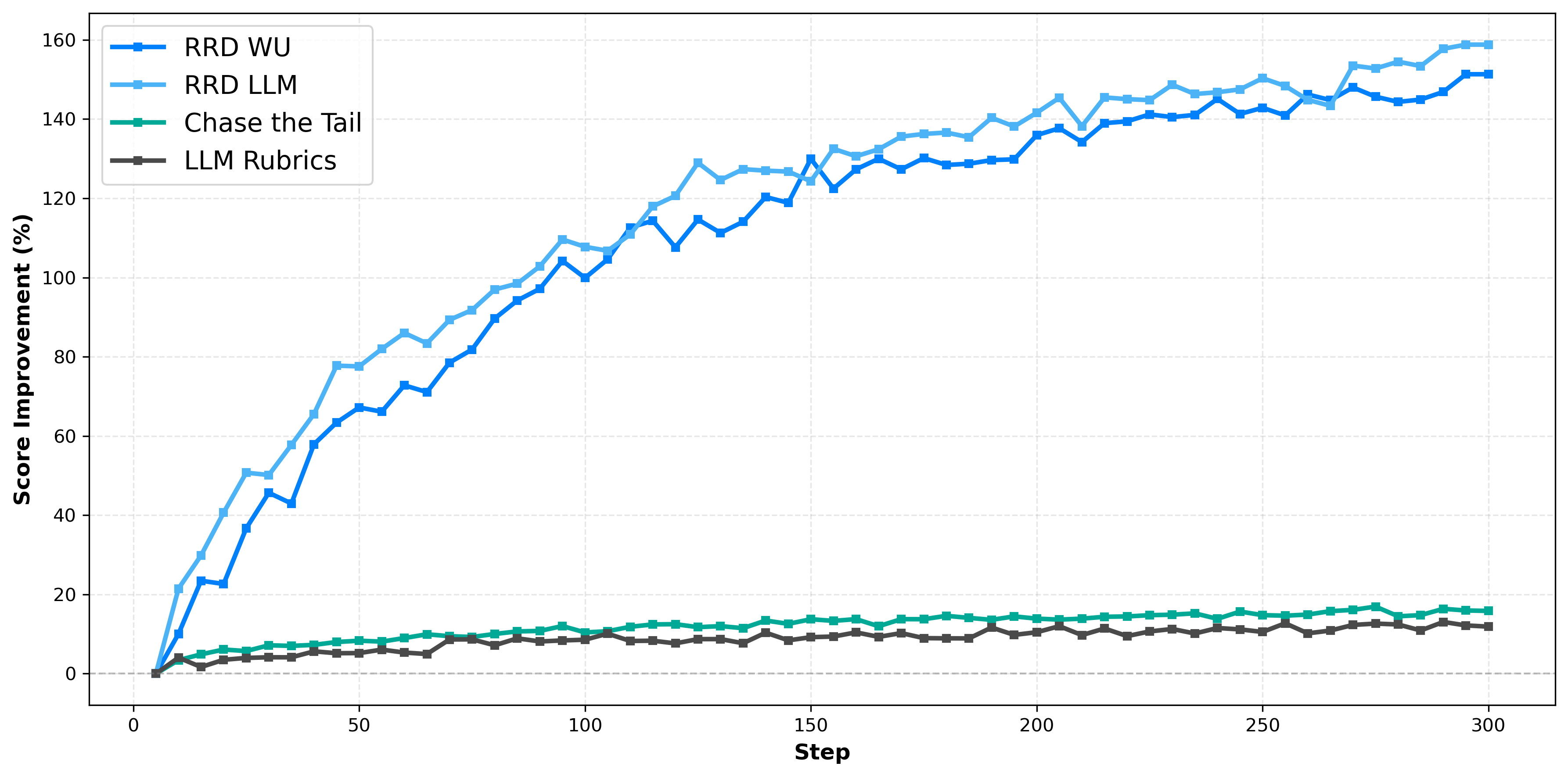}
\end{subfigure}%
\hfill
\begin{subfigure}{0.48\linewidth}
  \centering
  \includegraphics[width=\linewidth]{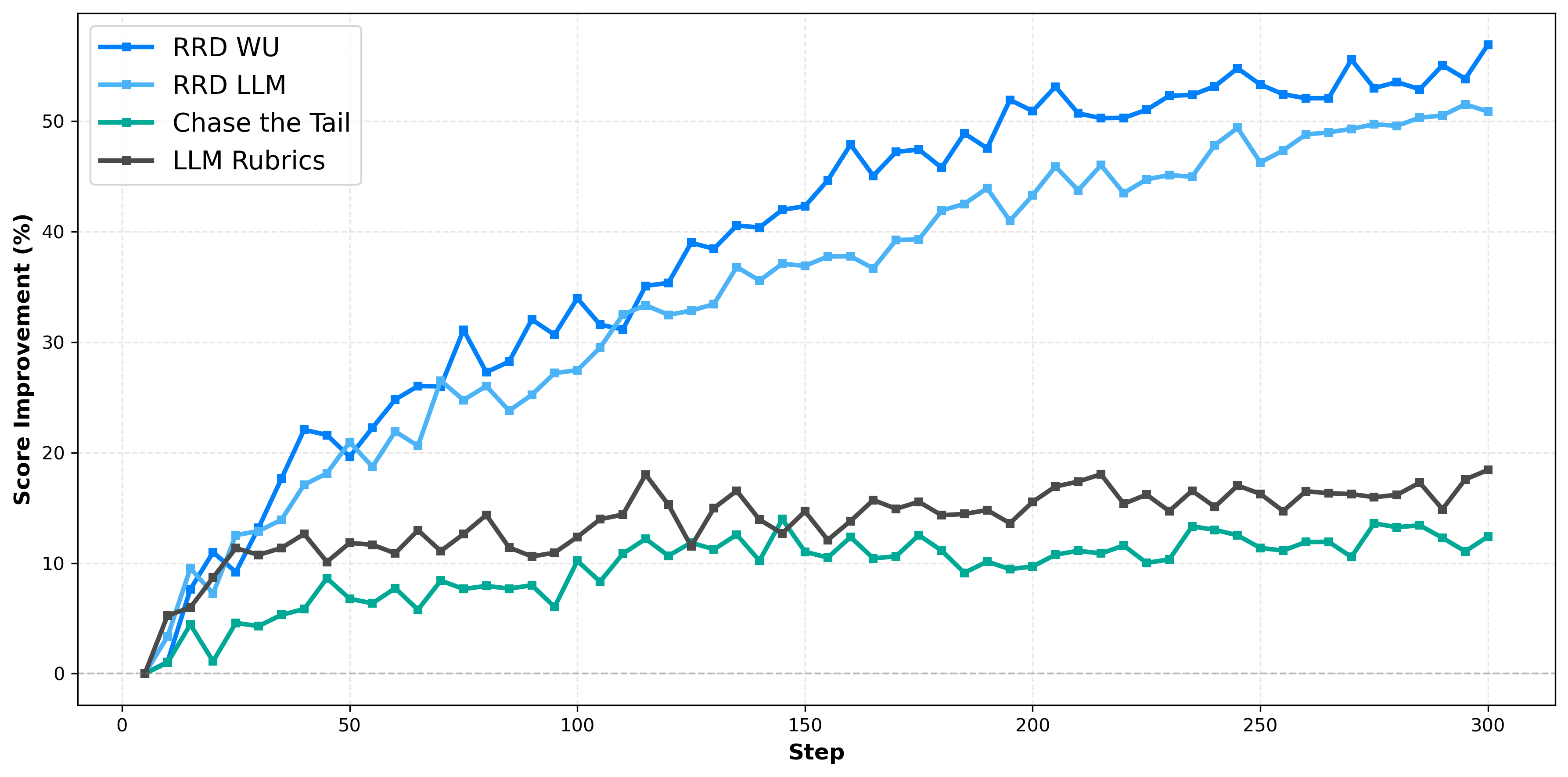}
\end{subfigure}
\caption{
Reward improvement during training of Qwen3-4B (left) and Llama3.1-8B-instruct (right) models using various rubric generation methods. Both $\ourmethod_\text{WU}$ and $\ourmethod_\text{LLM}$ provide a significantly stronger reward signal than traditional rubric-based or iterative baselines. $\ourmethod_\text{WU}$, in particular, shows superior training stability and higher cumulative reward gains across both architectures, indicating a more robust and granular supervision signal for RFT.
}
\label{fig:training_reward_curve}
\end{figure}
\input\begin{figure}[t]
  \centering
  \includegraphics[width=0.99\textwidth]{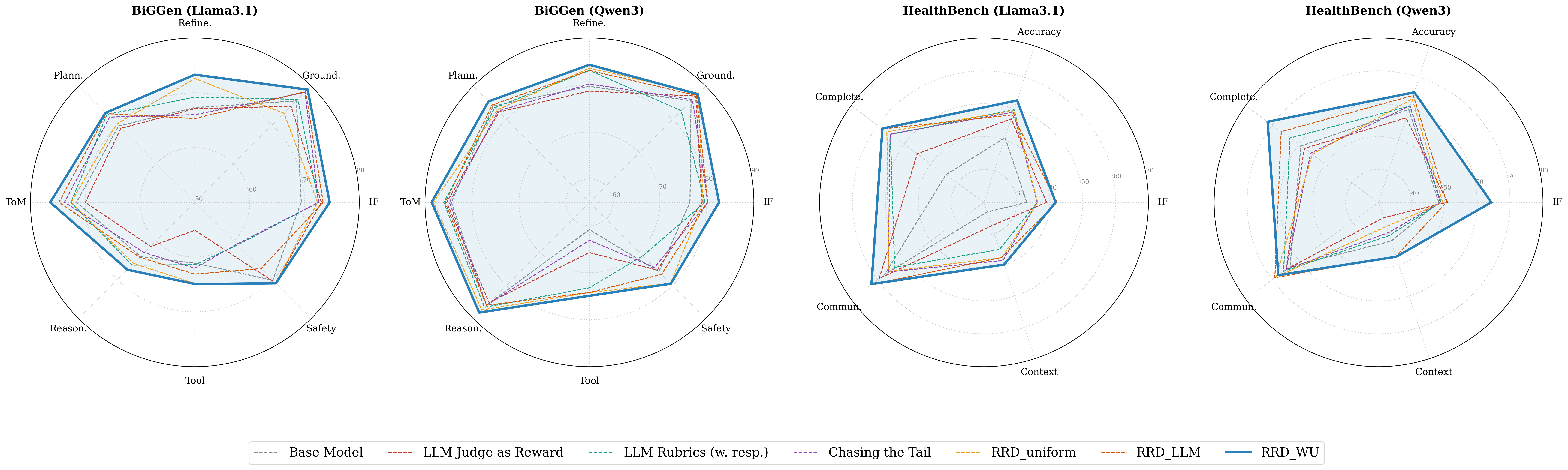}
  \caption{Multi-dimensional evaluation (scores in percentage) on BiGGen Bench (left) and HealthBench-Hard (right). Comparison of $\text{RRD}_{\text{WU}}$ against five baseline methods using Llama-3.1 and Qwen3 base models. $\text{RRD}_{\text{WU}}$ (solid red) demonstrates robust improvements across all axes, particularly in Instruction Following (IF) and Completeness.}
  \label{fig:radar}
\end{figure}
\paragraph{Reward Dynamics during Training.} 
Figure~\ref{fig:training_reward_curve} shows reward improvement during RFT for Qwen3-4B (left) and Llama3.1-8B-Instruct (right) under different rubric-generation strategies. Across both models, \ourmethod produces a substantially stronger learning signal: rewards climb rapidly in the first $\sim$50--100 steps and continue improving throughout training, while prior baselines plateau early at low gains. Concretely, on Qwen3-4B, $\ourmethod_{\mathrm{WU}}$ reaches roughly 150--160\% reward improvement by the end of training, compared to only $\sim$10--20\% for LLM Rubrics and Chasing the Tail; on Llama3.1-8B-Instruct, it attains about 55--60\% versus $\sim$10--20\% for baselines. Both RRD variants outperform alternatives, with $\ourmethod_{\mathrm{WU}}$ also showing the smoothest, most stable curves. Overall, the large absolute gaps (often $\geq$3--10$\times$ higher improvements) indicate that \ourmethod yields more discriminative and better-calibrated rewards, enabling sustained optimization rather than early saturation on open-ended tasks.

\paragraph{Policy Performance on BiGGen Bench.} The superior reward signal provided by \ourmethod translates into state-of-the-art performance for the resulting policies across diverse generation capabilities, as shown in the radar chart (Figure \ref{fig:radar}) with more details breakdown in Table~\ref{tab:biggenbench_main} (Appendix \ref{app:rft_training_res}). For both Qwen3-4B and Llama3.1-8B backbones, $\ourmethod_\text{WU}$ achieves the highest overall scores of $82.8\%$ and $71.1\%$, respectively. The method is particularly effective at improving instruction following (IF), refinement, and reasoning, all while maintaining robust performance on safety axes.

\paragraph{Generalization to High-Stakes Domains.} The performance is further confirmed by looking at the evaluation on the HealthBench dataset as shown in the same radar chart (Fig. \ref{fig:radar}) with more details breakdown in Table~\ref{tab:healthbench_main} (Appendix \ref{app:rft_training_res}) too. These results confirm that the benefits of \ourmethod transfer effectively to high-stakes, domain-specific tasks, such as the medical field. On HealthBench-Hard, $\ourmethod_\text{WU}$ consistently yields the best policy performance, particularly in IF, accuracy, and completeness ($+16.0\%$, $+5.5\%$, and $+12.5\%$ points respectively on the Qwen3-4B backbone). These results underscore the framework's ability to provide granular supervision that aligns with complex, physician-authored evaluation criteria.







\section{Related Works} \label{sec:related}

\paragraph{LLM-as-a-Judge.} 
Early work has shown that holistic judges, those producing a single verdict, can approximate human preferences with reasonable correlation~\citep{zheng2023judging, liu2023g, dubois2024length, wang2023pandalm, bavaresco2025llms}. However, such approaches suffer from bias~\citep{pezeshkpour2023large, saito2023verbosity}, inconsistency~\citep{zheng2023judging, haldar2025rating}, and opacity~\citep{kim2023prometheus, gajcin2025interpreting}, especially given the subjective, nuanced, and multidimensional nature of open-ended evaluation~\citep{thakur2024judging}. To address these limitations, recent work has shifted toward rubric-assisted judges~\citep{kim2023prometheus, hashemi2024llm, kim2025rubric, arora2025healthbench}, though many rely on static or heuristic rubrics that either lack scalability or prompt-specific nuance. In contrast, we propose a dynamic approach that recursively decomposes evaluation criteria to ensure both broad coverage and discriminative power.

\paragraph{Rubrics-based Rewards}
The use of structured rubrics has recently expanded beyond evaluation to serve as reward signals for training~\citep{gunjal2025rubrics}. To improve rubric quality, \citet{xie2025auto} introduce a Propose–Evaluate–Revise paradigm, while \citet{zhang2025chasing} optimize rubrics by maximizing score differentials between high-quality responses. Domain-specific approaches, such as \citet{wang2025infimed}, condition rubric generation on prompt context and retrieved exemplars, producing both positive and negative criteria. Other works leverage rubrics as priors for data generation or policy learning~\citep{huang2025reinforcement, zhou2025breaking}, or explore online rubric generation~\citep{rezaei2025online}. \ourmethod differs from prior work by introducing explicit mechanisms that enforce rubric \emph{informativeness}, \emph{comprehensiveness}, and \emph{non-redundancy} in LLM-generated criteria, and demonstrates its effectiveness as a source of high-quality RFT reward signals.



\paragraph{Reinforcement Fine-Tuning (RFT) for Open-Ended Tasks.} Reinforcement Learning from Verifiable Rewards (RLVR) has shown immense success in domains with objective ground truth, such as mathematics and coding, exemplified by models like DeepSeek-R1 and architectures utilizing process rewards~\citep{guo2025deepseek, su2025crossing, wen2025reinforcement}. Extending this success to open-ended tasks (e.g., creative writing, brainstorming) remains an open challenge due to the lack of ground-truth verifiers~\citep{zhang2025auditable, simonds2025self}. Recent works have attempted to bridge this gap by using LLM judges as proxy reward models for algorithms like GRPO. Our work contributes to this frontier by providing a ``Generative Verifier'' via \ourmethod that offers the granularity and reliability required to stabilize RFT in subjective and open-ended domains.
\section{Conclusion}
In this work, we introduced Recursive Rubric Decomposition (\ourmethod), a principled framework for generating informative, comprehensive and non-redundant rubrics for rubric-based LLM judges and reward modeling for open-ended tasks. \ourmethod operates by recursively decomposing high-level evaluation criteria into granular, discriminative rubrics, filtering for positive edge and non-redundancy to ensure reliability while optimizing weights to account for correlation structures and prevent the over-representation of highly correlated metrics. Empirically, we demonstrate that \ourmethod significantly improves the accuracy of LLM judges in a training-free setting, and establishes its effectiveness as a high-fidelity reward model for RFT. Policies trained with \ourmethod-derived rewards exhibit stronger alignment with human preferences on complex, open-ended generation tasks. These results highlight the importance of structured, granular, and statistically grounded rubric generation as a critical pathway toward scalable, interpretable, and reliable alignment for the next generation of LLMs.


\bibliographystyle{plainnat}
\bibliography{reference}

\appendix
\newpage
\section*{Appendix}
\label{sec:appendix}

\section{Algorithm}\label{app:algo}
\begin{algorithm}[H]
\caption{\textbf{RRD: Recursive Rubric Decomposition}}
\label{alg:rrd}
\begin{algorithmic}

\State \textbf{Require:} Let $\mathcal{P}$ be the prompt set; $\{R_i\}_{i=1}^{n} \in \mathcal{R}$ be sampled responses for each prompt $P \in \mathcal{P}$; $\Psi$ be the LLM-based rubric proposer. 

\State \textbf{Procedure:}
\For {each $P \in \mathcal{P}$:}
\State Let $\Psi$ generate an initial rubric list $\mathcal{G}_0 = \Psi(P, \{R_i\}_{i=1}^{n})$, 
\vspace{0.5em}

\State Initialize $\mathcal{G}_{\text{final}} \leftarrow \mathcal{G}_0$, iteration counter $t \leftarrow 0, |\mathcal{G}_{\text{filtered}}| \leftarrow 0$

\While{$|\mathcal{G}_{\text{filtered}}| < N$}
    \State $t \leftarrow t + 1$
    \For{each rubric $g_m \in \mathcal{G}_{t-1}$} \hfill \textit{Step 1: Rubric Evaluation}
        \State $\mathcal{R}_m \gets \{\, R_i \in \{R_i\}_{i=1}^{n} : g_m(P,R_i)=1 \,\}$

        
            \If{$|\mathcal{R}_m| \ge n$} \hfill \textit{Step 2: Rubric Decomposition}
                \State $\mathcal{G}^{\text{new}}_m \leftarrow \Psi(g_m, \mathcal{R}_m)$ 
                \State $\mathcal{G}_{t-1} \leftarrow \mathcal{G}_{t-1} \cup \mathcal{G}^{\text{new}}_m$
            \EndIf
        \EndFor
    \State $\mathcal{G}_{t}\gets \{\, m\in\mathcal{G}_{t-1}\mid \nexists\, m'\neq m:\ \mathrm{conflict}(m,)\vee \mathrm{overlap}(m,m') \,\}$  \hfill \textit{Step 3: Rubric Filter}

    \State $\mathcal{G}_{\text{final}} \leftarrow \mathcal{G}_{t}$
\EndWhile
\EndFor

\State \Return $\mathcal{G}_{\text{final}}$
\end{algorithmic}
\end{algorithm}

\newpage
\section{Additional Notes on Derivations and Proofs} \label{app:proofs}

\subsection{Probability Upper Bound for Incorrect Rubric-based Judge Verdict (Eq.\ref{eq:upper_bound_prob_wrong_verdict})}\label{app:proof_upper_bound_prob_wrong_verdict}

\textit{Derivation.}
Condition on the class $Y=y$, an error is $\{\widehat Y\neq Y\} \iff \{\Gamma\le 0\}$. Let $W_y=\mathbf{w}^\top Z$,
$$
\Pr(\Gamma\le 0\mid Y=y)
= \Pr(W_y \le -\Delta_m \mid Y=y).
$$

For a sub-Gaussian, it is trivial to show that 
\[
\Pr(X\le -a)=\Pr(e^{-\lambda X}\ge e^{\lambda a})\le \mathbb E[e^{-\lambda X}]e^{-\lambda a}\le \exp(\tfrac{\lambda^2\sigma^2}{2}-\lambda a).
\]

For $\lambda>0$, 
$$
\Pr(X\le -a)\ \le\ \exp\!\Big(-\frac{a^2}{2\sigma^2}\Big).
$$

Substitute $a$ with $\Delta_m$ and $\sigma^2$ with $=V_m(y)$, 
$$
\Pr(\Gamma\le 0\mid Y=y)
= \Pr(W_y \le -\Delta_m \mid Y=y)
\le \exp\!\Big(-\frac{\Delta_m^2}{2V_m(y)}\Big),
$$

Therefore, by the law of total probability,
$$
\Pr(\widehat Y\neq Y)
= \pi_+\Pr(\Gamma\le 0\mid Y=+1)+\pi_-\Pr(\Gamma\ge 0\mid Y=-1),
$$

Hence,
\begin{align*}
\Pr(\widehat Y\neq Y)
&\ \le\ \pi_+\,e^{-\Delta_m^2/(2V_m(+1))}\ +\ \pi_-\,e^{-\Delta_m^2/(2V_m(-1))} \\
&\ \le\ \max\!\Big\{e^{-\Delta_m^2/(2V_m(+1))},\ e^{-\Delta_m^2/(2V_m(-1))}\Big\} \\
&\ \le\ \exp\!\big(-\min\{\Delta_m^2/(2V_m(+1)),\,\Delta_m^2/(2V_m(-1))\}\big),
\end{align*}

\textit{Now, we complete the derivation for Eq.\ref{eq:upper_bound_prob_wrong_verdict}.} 

\vspace{2em}

\textit{Note.} As a side note following Eq.\ref{eq:upper_bound_prob_wrong_verdict}, the motivation for expanding the set of rubrics with positive information edge is clearer under the \textit{idealized case} where the rubrics are \textit{equicorrelated, equal-variance, equal-weight}.


Recall $V_m(Y)=\operatorname{Var}\!\left(\sum_{k=1}^m w_k Z_k \,\middle|\, Y\right)$. With $w_k=1$,
\begin{align*}
V_m(Y)&=\operatorname{Var}\!\Big(\sum_{k=1}^m Z_k \,\Big|\, Y\Big) \\
&= \sum_{k=1}^m \operatorname{Var}(Z_k\mid Y)\;+\;2\sum_{1\le i<j\le m}\operatorname{Cov}(Z_i,Z_j\mid Y).
\end{align*}
Plugging in the equicorrelation/equal-variance values:
\begin{align*}
\sum_{k=1}^m \operatorname{Var}(Z_k\mid Y)&= m\,\sigma^2, \\
\operatorname{Cov}(Z_i,Z_j\mid Y)&=\rho\,\sigma^2\ \ (i\neq j).
\end{align*}
There are $\binom{m}{2}=\frac{m(m-1)}{2}$ distinct pairs, so
\begin{align*}
V_m(Y) &= m\,\sigma^2 \;+\; 2\cdot \frac{m(m-1)}{2}\,\rho\,\sigma^2 \\
&= \sigma^2\big[m + (m^2-m)\rho\big].
\end{align*}
Under the symmetric assumptions of the corollary, this does not depend on $Y=+1$ or $Y=-1$, so we simply write
$$
V_m \;=\; \sigma^2\big[m + (m^2-m)\rho\big].
$$

As such, with the idealized assumptions of $w_k\equiv 1$, $\mu_k\equiv\mu>0$, $\mathrm{Var}(S_k\mid Y)=\sigma^2$, and $\mathrm{Corr}(S_i,S_j\mid Y)=\rho$ ($i\ne j$),
\[
\mathbb{P}(\widehat Y\ne Y)\ \le\ \exp\!\left(-\frac{m\,\mu^2}{2\,\sigma^2\,[1+(m-1)\rho]}\right).
\]

Following this, it becomes obvious that when rubrics are uncorrelated (i.e., $\rho$ = 0), the prediction error can be consistently reduced by consistently expanding the rubric pool, provided that each added rubric contributes a positive information edge (i.e., its prediction direction aligns with human preference). \textbf{This supports the recursive process to encourage more extensive rubric search for broader coverage}. In practice, however, rubrics are neither independent nor equicorrelated. When the exact information edge $\mu$ is \textit{unobservable} and \textit{individual evaluation dimensions carry comparable importance} (i.e., no single criterion dominating the overall judgment as often the case for open-ended tasks), \textbf{the aggregation of extended rubric set must explicitly account for the rubric correlation structure}. This motivates incorporating correlation-aware normalization via an effective weighting scheme, which we introduce below.

\subsection{Rubric Weighting in Whitened Space}\label{app:weighting_in_whitened_space}
Define the ratio to be optimized as:
\[
\Xi = {(\vect{w}^\top \vect{\mu})^2}/({\vect{w}^\top \vect{\Sigma} \vect{w}})
\]

Also define whitened coordinates:
\[
u \;:=\; \frac{\Sigma^{1/2}w}{\norm{\Sigma^{1/2}w}_2}\in\R^m, 
\qquad 
z \;:=\; \frac{\Sigma^{-1/2}\mu}{\norm{\Sigma^{-1/2}\mu}_2}\in\R^m, 
\qquad 
\kappa:=\norm{\Sigma^{-1/2}\mu}_2^2.
\]
\begin{lemma}\label{lem:cosine}
For any $w\neq 0$,
\[
\Xi(w)\;=\;\kappa\;\big\langle u,\,z\big\rangle^2
\qquad\text{and}\qquad
\frac{\Xi(w)}{\Xi(w^\star)} \;=\; \cos^2\!\angle\!\big(\Sigma^{1/2}w,\;\Sigma^{-1/2}\mu\big)\in[0,1].
\]
\end{lemma}

\begin{proof}
By direct algebra:
$(w^\top\mu)^2 = \big\langle \Sigma^{1/2}w,\,\Sigma^{-1/2}\mu\big\rangle^2
= \norm{\Sigma^{1/2}w}_2^2\norm{\Sigma^{-1/2}\mu}_2^2\ip{u}{z}^2$,
and $w^\top\Sigma w=\norm{\Sigma^{1/2}w}_2^2$. Thus, $\Xi(w^\star)=\kappa$, where $w^\star$ is the optimized weight.
\end{proof}

\begin{lemma}
\label{lemma:rubric_covariance_convergence}
Let $X \in \R^m$ be a zero-mean rubric-score vector with covariance 
$\Sigma = \E[XX^\top] \succ 0$ (Note: $X$ is sub-Gaussian given bounded rubric scores in $[0,1]$
and $X_i \ for \ i \in [1,N]$ are i.i.d. copies of $X$).

Define the sample covariance 
\[
\widehat{\Sigma} := \frac{1}{N}\sum_{i=1}^N X_i X_i^\top,
\qquad
r_{\mathrm{eff}}(\Sigma) := 
\frac{\tr(\Sigma)}{\|\Sigma\|_{\mathrm{op}}} \in [1,m].
\]
Then there exist constants $c, \ C > 0$ such that, for all $t \ge 0$,
\[
\ \big\|\widehat{\Sigma} - \Sigma\big\|_{\mathrm{op}}
\ \le\
C\,\|\Sigma\|_{\mathrm{op}}
\!\left(
\sqrt{\frac{r_{\mathrm{eff}}(\Sigma)}{N}}
+\frac{r_{\mathrm{eff}}(\Sigma)}{N}
+\frac{t}{\sqrt{N}}
\right)
\quad\text{with probability at least } 1 - 2e^{-c t^2}.
\]
In particular,
\[
\E\,\|\widehat{\Sigma} - \Sigma\|_{\mathrm{op}}
\;\le\;
C\,\|\Sigma\|_{\mathrm{op}}\!
\left(
\sqrt{\frac{r_{\mathrm{eff}}(\Sigma)}{N}}
+\frac{r_{\mathrm{eff}}(\Sigma)}{N}
\right).
\]

Moreover, if $\lambda_{\min}(\Sigma) > 0$ and
$\|\widehat{\Sigma} - \Sigma\|_{\mathrm{op}} \le \tfrac{1}{2}\lambda_{\min}(\Sigma)$, then

\[
\big\|\widehat{\Sigma}^{-1/2} - \Sigma^{-1/2}\big\|_{\mathrm{op}}
\ \le\
\tfrac{1}{2}\,\lambda_{\min}(\Sigma)^{-3/2}\,
\big\|\widehat{\Sigma} - \Sigma\big\|_{\mathrm{op}},
\]

Thus whitened uniform weights computed from $\widehat{\Sigma}$ are consistent and numerically stable.
\end{lemma}

\begin{proof}
Let $Y_i := X_i X_i^\top - \Sigma$, estimation error for or covariance $\Sigma$ is 
$
\widehat{\Sigma} - \Sigma 
= \frac{1}{N}\sum_{i=1}^N X_i^\top X_i - \Sigma
= \frac{1}{N}\sum_{i=1}^N Y_i
$ 
with $\E Y_i = 0$. 
To control the random fluctuation $\widehat{\Sigma} - \Sigma$ in operator norm, we will apply a matrix Bernstein inequality to the sum $\sum^N_{i=1} Y_i$. This requires bounding the matrix variance proxy
$
V := \Big\|\sum_{i=1}^N \E\,Y_i^2\Big\|_{\mathrm{op}}
$,
which captures the second-order size of the fluctuations.

Since
$Y_i^2 = (X_i X_i^\top - \Sigma)^2
= X_i X_i^\top X_i X_i^\top - X_i X_i^\top \Sigma - \Sigma X_i X_i^\top + \Sigma^2$,
we have
$
\E\,Y_i^2
= \E[XX^\top XX^\top] - \Sigma^2
$. 

Thus, for any unit $u$,
$
u^\top \E[XX^\top XX^\top]\,u
= \E[\|X\|_2^2\,(u^\top X)^2]
\le \sqrt{\E\|X\|_2^4}\,\sqrt{\E(u^\top X)^4}
$.
Sub-Gaussian moment bounds give
$\E\|X\|_2^4 \le C (\tr \Sigma)^2$ and
$\E(u^\top X)^4 \le C (u^\top \Sigma u)^2 \le C \|\Sigma\|_{\mathrm{op}}^2$, where $C$ is universal constant.
Plugging this back we get
\[
u^\top \E[XX^\top XX^\top]\,u
\le C\,\|\Sigma\|_{\mathrm{op}}\,\tr(\Sigma),
\qquad
\Rightarrow\quad
\big\|\E[XX^\top XX^\top]\big\|_{\mathrm{op}}
\le C\,\|\Sigma\|_{\mathrm{op}}\,\tr(\Sigma).
\]
Therefore,
\[
V = \Big\|\sum_{i=1}^N \E Y_i^2\Big\|_{\mathrm{op}}
\le N\,C\,\|\Sigma\|_{\mathrm{op}}\,\tr(\Sigma)
= C\,N\,\|\Sigma\|_{\mathrm{op}}^2\,r_{\mathrm{eff}}(\Sigma),
\]
where $r_{\mathrm{eff}}(\Sigma)={\tr(\Sigma)}/{\|\Sigma\|_{\mathrm{op}}}$.


To apply the non-commutative Bernstein
inequality, we also require a uniform sub-exponential bound on the size of the
summands, encoded by a parameter $L>0$ such that $\|Y_i\|_{\op}$ is
sub-exponential with $\|Y_i\|_{\psi_1} \le L$. Since
\[
\|Y_i\|_{\op}
= \|X_i X_i^\top - \Sigma\|_{\op}
\le \|X_i X_i^\top\|_{\op} + \|\Sigma\|_{\op}
= \|X_i\|_2^2 + \|\Sigma\|_{\op},
\]
and $X$ is sub-Gaussian, it follows that $\|Y_i\|_{\op}$ is sub-exponential with
$\|Y_i\|_{\psi_1} \le C \|\Sigma\|_{\op}$ for some universal $C$. With $V$ and $L=C \|\Sigma\|_{\op}$, we now invoke the intrinsic-dimension version of the matrix Bernstein inequality for the sum $\sum_{i=1}^N Y_i$ and have:
\[
\mathbb{P}\!\left(
\Big\|\frac{1}{N}\sum_{i=1}^N Y_i\Big\|_{\mathrm{op}} \ge s
\right)
\le
2\exp\!\left\{
-\,c\,\min\!\left(
\frac{N^2 s^2}{V},\, \frac{N s}{L}
\right)
\right\},
\]
where $V$ and $L$ are as above and $c>0$ is a universal constant.

Substituting $V \le C N \|\Sigma\|_{\mathrm{op}}^2 r_{\mathrm{eff}}(\Sigma)$ yields
\[
\mathbb{P}\!\left(
\big\|\widehat{\Sigma} - \Sigma\big\|_{\mathrm{op}} \ge s
\right)
\le
2\exp\!\left\{
-c N\,\min\!\left(
\frac{s^2}{C \|\Sigma\|_{\mathrm{op}}^2 r_{\mathrm{eff}}(\Sigma)},\,
\frac{s}{C \|\Sigma\|_{\mathrm{op}}}
\right)
\right\}.
\]
Choosing
$s = C\,\|\Sigma\|_{\mathrm{op}}
\!\left(\sqrt{r_{\mathrm{eff}}/N} + r_{\mathrm{eff}}/N + t/\sqrt{N}\right)$
gives, after rescaling constants,
\[
\big\|\widehat{\Sigma} - \Sigma\big\|_{\mathrm{op}}
\ \le\
C\,\|\Sigma\|_{\mathrm{op}}
\!\left(
\sqrt{\frac{r_{\mathrm{eff}}(\Sigma)}{N}}
+\frac{r_{\mathrm{eff}}(\Sigma)}{N}
+\frac{t}{\sqrt{N}}
\right),
\]
with probability at least $1 - 2e^{-c t^2}$.
Integrating this tail yields the stated expectation bound.

For symmetric positive definite $A, \ B$ with
$\lambda_{\min}(A), \lambda_{\min}(B) \ge \gamma > 0$
and any $C^1$ scalar function $f$ on $[\gamma,\infty)$,
the spectral calculus gives
\[
\|f(A) - f(B)\|_{\mathrm{op}}
\le
\sup_{t \in [\gamma, \infty)} |f'(t)|\, \|A - B\|_{\mathrm{op}}.
\]
Applying this with $f(t) = t^{-1/2}$ 
yields
$\sup_{t \ge \gamma} |f'(t)| = \tfrac{1}{2}\gamma^{-3/2}$. 
Taking $A = \widehat{\Sigma}$, $B = \Sigma$, and
$\gamma = \tfrac{1}{2}\lambda_{\min}(\Sigma)$
(valid when $\|\widehat{\Sigma} - \Sigma\|_{\mathrm{op}}
\le \tfrac{1}{2}\lambda_{\min}(\Sigma)$) gives
\[
\big\|\widehat{\Sigma}^{-1/2} - \Sigma^{-1/2}\big\|_{\mathrm{op}}
\le
\tfrac{1}{2}\lambda_{\min}(\Sigma)^{-3/2}
\big\|\widehat{\Sigma} - \Sigma\big\|_{\mathrm{op}},
\]
\end{proof}

\begin{theorem} \label{thm:wu}
Let the weight space be
$
\mathcal{W}:=\bigl\{\,w\in\mathbb{R}^m:\ \left\lVert {\Sigma}^{1/2}w \right\rVert_2=1\,\bigr\}
$.
Then
\[
\sup_{w\in\mathcal{W}}\ \inf_{\mu\in\mathcal{U}}\ \mathrm{SNR}(w)
\quad\text{is attained by}\quad
w_{\text{wu}}\ \propto\ {\Sigma}^{-1/2}\mathbf{1}
\]  
\end{theorem}

\begin{proof}
Assume whitened edge $\Sigma^{-1/2}\mu$ is nonnegative and exchangeable in distribution across coordinates. We fix its magnitude $\norm{\Sigma^{-1/2}\mu}_2^2=c>0$ but leave its direction unknown.

By Lemma~\ref{lem:cosine}, $\Xi(w)=c\,\ip{u}{z}^2$ with $u=\Sigma^{1/2}w/\norm{\Sigma^{1/2}w}_2$ and $z=\Sigma^{-1/2}\mu/\norm{\Sigma^{-1/2}\mu}_2$.
Under the assumptions, $z$ is supported on the positive orthant and is exchangeable with fixed norm; the \emph{least favorable} $z$ against a given $u$ is the one minimizing $\ip{u}{z}^2$.
By symmetry and convexity of the feasible set of $z$, the $u$ that maximizes the worst-case inner product is the barycenter of the positive orthant, i.e., $u\propto \1$. 
Translating back gives $w\propto \Sigma^{-1/2}\1$. 
\end{proof}

\section{Discussion}
\label{app:weak_strong_discussion}
\paragraph{Misalignment filtering discussion.}
We include an directionality guardrail that flags rubrics whose induced preferences systematically favor a weaker reference model over a stronger one. This heuristic is motivated by prior evidence that \emph{strong} LLM judges can, on average, closely track human pairwise preferences on open-ended prompts (e.g., GPT-4 achieving human-comparable agreement in MT-Bench/Chatbot Arena; \citep{zheng2023judging}), and that structured LLM-based evaluation frameworks can substantially improve correlation with human judgments in NLG evaluation \citep{liu2023g}. More broadly, using a more capable model as a source of supervision is a standard alignment pattern (e.g., AI feedback replacing or complementing human labels in Constitutional AI and related RLAIF setups; \citep{bai2022constitutional,lee2023rlaif}). Importantly, we do \emph{not} treat “stronger model preferred” as a normative definition of quality: the heuristic can fail on value-sensitive axes where humans may prefer caution, brevity, calibrated uncertainty, or certain refusal behaviors. Accordingly, we treat the filter as a conservative sanity-check rather than a core dependency of RRD, and recommend disabling it in such domains or replacing it with (i) a multi-reference check across several strong models and/or (ii) axis-specific exemptions/calibration for known inversion-prone criteria.

\paragraph{Controlling over-decomposition.}
A potential failure mode of recursive rubric refinement is \emph{over-decomposition}: when the decomposition trigger is aggressive (e.g., small coverage threshold), the procedure may fragment criteria into overly specific sub-rubrics that track incidental artifacts of the sampled responses (e.g., stylistic quirks) rather than stable preference dimensions. To mitigate this, RRD includes two complementary safeguards. First, we apply a \emph{non-redundancy} filter that removes candidate rubrics that are duplicative, conflicting, or near-paraphrases of existing ones, preventing correlated or semantically equivalent sub-rubrics from accumulating and effectively ``double-counting'' the same dimension. Second, we impose a \emph{rejection-based early stopping} criterion by setting up a tunable hyperparameter \emph{termination threshold}:  when a decomposition round produces too many invalid or non-novel candidates (exceeding a rejection threshold), we halt further recursion. Together, these mechanisms bound rubric set growth and reduce the risk that continued recursion degenerates into overly fine-grained, sample-specific criteria, while preserving the benefits of decomposition when it reveals genuinely discriminative sub-dimensions.

\newpage
\section{RFT Training Details}\label{app:rft_training_details}

Reinforcement Fine-Tuning (RFT) aligns a policy model $\pi_{\theta}$ by maximizing expected return under prompts $q\sim p_Q$ and model-generated responses $o\sim \pi_{\theta}(\cdot\mid q)$:
\[
J(\pi_\theta)
=\mathbb{E}_{q\sim p_Q}\Big[\mathbb{E}_{o\sim \pi_\theta(\cdot\mid q)}[R(q,o)]
-\beta\,D_{\mathrm{KL}}\!\big(\pi_\theta(\cdot\mid q)\,\|\,\pi_{\mathrm{ref}}(\cdot\mid q)\big)\Big],
\]
where $R(q,o)=\sum_{t=1}^{|o|} r(s_t,o_t)$ is the trajectory return and $\pi_{\mathrm{ref}}$ is an optional reference policy.

In practice, policy-gradient methods such as PPO \citep{schulman2017proximal} optimize a clipped surrogate objective using an importance ratio against the pre-update policy $\pi_{\theta_{\text{old}}}$:
\[
J_{\mathrm{PPO}}(\pi_\theta)
=\mathbb{E}_{q\sim p_Q,\,o\sim \pi_{\theta_{\text{old}}}(\cdot\mid q)}
\sum_{t=1}^{|o|}
\min\!\Big(r_t(\theta)\hat A_t,\ \mathrm{clip}(r_t(\theta),1-\epsilon,1+\epsilon)\hat A_t\Big),
\]
with $r_t(\theta)=\frac{\pi_\theta(o_t\mid q,o_{<t})}{\pi_{\theta_{\text{old}}}(o_t\mid q,o_{<t})}$. GRPO \citep{shao2024deepseekmath} removes the need for a learned value function by sampling a group of $G$ responses $\{o_i\}_{i=1}^G$ per question and using a group-relative advantage for all tokens in response $o_i$:
\[
\hat A_i
=\frac{R(q,o_i)-\mathrm{mean}(\{R(q,o_j)\}_{j=1}^G)}{\mathrm{std}(\{R(q,o_j)\}_{j=1}^G)}.
\]
However, the standard GRPO objective additionally aggregates token losses with a per-response length normalization factor $1/|o_i|$, and together with per-question $\mathrm{std}(\cdot)$ normalization this induces an optimization bias. Dr.GRPO \citep{drgrpo} mitigates this bias by removing the $1/|o_i|$ token-aggregation term and removing the per-question $\mathrm{std}(\cdot)$ normalization, yielding an unbiased policy-gradient estimator and improved token efficiency.

In our experiments, we consistently apply Dr.GRPO algorithm in RFT training across baselines and \ourmethod variants. We provide RFT training parameters as below in Table \ref{tab:rl-hparams}.
\begin{table}[ht]
\centering
\begin{threeparttable}
\caption{Shared RFT training hyperparameters. We use the \texttt{verl} library \citep{sheng2025hybridflow} for RFT.}
\label{tab:rl-hparams}
\begin{tabular}{ll}
\toprule
\textbf{Parameter} & \textbf{Value} \\
\midrule
Algorithm & Dr,GRPO \\
Rollouts per prompt & 8 \\
Learning rate & $1 \times 10^{-6}$ \\
Learning rate schedule & Constant with no warmup \\
Global batch size & 128 \\
Reward-level KL coefficient & $1 \times 10^{-3}$ \\
Max. training steps & 1000 \\
Max. gen. tokens & 2048 \\
Training GPUs & 8$\times$ NVIDIA H100s \\
$\pi_J$ & GPT-OSS-120B~\citep{agarwal2025gpt} \\
Optimizer & AdamW~\citep{loshchilov2017decoupled} \\
Parallelism strategy & FSDP~\citep{rajbhandari2020zero} \\
\bottomrule
\end{tabular}
\begin{tablenotes}\footnotesize
\item Note: We use the PyTorch FSDP implementation; see \url{https://pytorch.org/docs/stable/fsdp.html}.
\end{tablenotes}
\end{threeparttable}
\end{table}



\newpage
\section{RFT Training Results}\label{app:rft_training_res}

\begin{table}[h]

\centering

\small

\caption{Performance (in percentage) on BiGGen Bench of the base models (pre-RFT) and models fine-tuned with various methods. $\ourmethod_\text{WU}$ (highlighted in gray) consistently improves performance across both Qwen3-4B and Llama3.1-8B backbones, yielding gains across evaluation axes as well as the overall score.}

\label{tab:biggen-cap}


\newcolumntype{C}[1]{>{\centering\arraybackslash}p{#1}}

\newcommand{\colW}{1.05cm} 

\resizebox{\linewidth}{!}{%

\begin{tabular}{l l *{8}{C{\colW}} c}

\toprule



\textbf{Model} & \textbf{Method} & \textbf{IF} & \textbf{Ground.} & \textbf{Refine.} & \textbf{Plann.} & \textbf{ToM} & \textbf{Reason.} & \textbf{Tool} & \textbf{Safety} & \textbf{Overall} \\

\midrule

\multirow{7}{*}{\textbf{Qwen3}}

& Base Model            & 76.4 & 85.6 & 79.7 & 84.0 & 84.4 & 86.0 & 60.8 & 75.9 & 77.9 \\

& LLM Judge as Reward      & 79.0 & 87.0 & 78.7 & 82.2 & 85.4 & 85.2 & 65.7 & 75.4 & 78.5 \\

& LLM Rubrics (w. resp.)          & 79.6 & 82.6 & 83.1 & 83.5 & 86.0 & 86.6 & 73.2 & 71.1 & 79.2 \\

& Chasing the Tail      & 80.2 & 86.0 & 80.2 & 82.5 & 84.8 & 86.0 & 63.1 & 74.6 & 78.6 \\

& $\ourmethod_{\text{uniform}}$ & 79.2 & 87.8 & 83.6 & 82.8 & 88.2 & 87.4 & 74.2 & 79.4 & 81.3 \\

& $\ourmethod_{\text{LLM}}$     & 80.2 & 87.2 & 83.1 & 84.3 & 85.7 & 86.0 & 74.2 & 76.7 & 80.6 \\

\rowcolor{verylightgray} & 

$\ourmethod_{\text{WU}}$      & \textbf{82.6} & \textbf{87.6} & \textbf{84.3} & \textbf{85.4} & \textbf{88.6} & \textbf{88.2} & \textbf{74.9} & \textbf{79.5} & \textbf{82.8} \\

\midrule

\multirow{7}{*}{\textbf{Llama3.1}}

& Base Model            & 69.4 & 76.2 & 67.3 & 69.7 & 71.6 & 64.0 & 61.1 & 70.1 & 67.0 \\

& LLM Judge as Reward      & 73.1 & 74.8 & 67.1 & 69.1 & 70.0 & 61.4 & 55.1 & 70.8 & 66.3 \\

& LLM Rubrics (w. resp.)          & 72.6 & 76.6 & 69.2 & 72.7 & 72.6 & 66.2 & 61.4 & 61.0 & 66.9 \\

& Chasing the Tail      & 72.6 & 78.4 & 66.0 & 72.0 & 73.8 & 63.0 & 62.0 & 60.8 & 67.3 \\

& $\ourmethod_{\text{uniform}}$ & 72.4 & 73.0 & 72.6 & 70.2 & 72.6 & 65.8 & 64.8 & 70.9 & 68.6 \\

& $\ourmethod_{\text{LLM}}$     & 73.4 & 78.6 & 65.3 & 72.9 & 74.8 & 64.2 & 63.1 & 67.1 & 68.0 \\

\rowcolor{verylightgray} & 

$\ourmethod_{\text{WU}}$      & \textbf{74.6} & \textbf{79.1} & \textbf{73.3} & \textbf{73.1} & \textbf{76.4} & \textbf{67.4} & \textbf{64.9} & \textbf{70.9} & \textbf{71.1} \\

\bottomrule

\end{tabular}%

}
\label{tab:biggenbench_main}
\end{table}

\begin{table}[h]

\centering

\small 

\setlength{\tabcolsep}{5pt}

\caption{Performance (in percentage) on HealthBench-Hard of the base model (pre-RFT) and models fine-tuned with various methods. $\ourmethod_\text{WU}$ outperforms both the base model and RFT baselines on key axes -- including instruction following (IF), accuracy, completeness, context, and the overall score -- when using Qwen3-4B as the backbone model. It also achieves comparable performance on the communication axis relative to other \ourmethod-variants. Similar performance gains are observed when using Llama3.1-8B as the backbone, demonstrating the robustness of $\ourmethod_\text{WU}$ across architectures.}

\label{tab:healthbench_main}


\newcolumntype{C}[1]{>{\centering\arraybackslash}p{#1}}

\newcommand{\hbColW}{1.8cm}

\resizebox{0.98\linewidth}{!}{%

\begin{tabular}{l l *{5}{C{\hbColW}} c}

\toprule



\textbf{Model} & \textbf{Method} & \textbf{IF} & \textbf{Accuracy} & \textbf{Completeness} & \textbf{Communication} & \textbf{Context} & \textbf{Overall} \\

\midrule

\multirow{7}{*}{\textbf{Qwen3}}

& Base Model                  & 48.3 & 59.7 & 59.2 & 63.3 & 42.4 & 55.9 \\

& LLM Judge as Reward      & 51.0 & 57.0 & 57.9 & 64.8 & 34.9 & 57.0 \\

& LLM Rubrics (w. resp.)& 49.0 & 60.8 & 63.3 & 65.7 & 40.6 & 59.3 \\

& Chasing the Tail        & 48.8 & 60.9 & 55.5 & 64.8 & 39.7 & 58.7 \\

& $\ourmethod_{\text{uniform}}$           & 49.5 & 63.1 & 54.9 & 68.6 & 37.7 & 60.1 \\

& $\ourmethod_{\text{LLM}}$               & 50.7 & 64.2 & 66.6 & 69.0 & 47.2 & 60.9 \\

\rowcolor{verylightgray} & 

$\ourmethod_{\text{WU}}$                 & \textbf{64.3} & \textbf{65.2} & \textbf{71.7} & \textbf{67.7} & \textbf{47.4} & \textbf{61.7} \\

\midrule

\multirow{7}{*}{\textbf{Llama3.1}}

& Base Model                  & 33.0 & 40.7 & 34.3 & 57.3 & 23.1 & 37.1 \\

& LLM Judge as Reward      & 39.0 & 46.7 & 45.1 & 59.4 & 27.5 & 44.1 \\

& LLM Rubrics (w. resp.)& 36.2 & 49.6 & 55.2 & 53.6 & 35.1 & 47.2 \\

& Chasing the Tail        & 36.2 & 48.9 & 55.3 & 56.0 & 38.6 & 47.5 \\

& $\ourmethod_{\text{uniform}}$           & 36.3 & 49.3 & 56.5 & 55.7 & 37.7 & 47.6 \\

& $\ourmethod_{\text{LLM}}$               & 42.2 & 48.0 & 58.0 & 62.5 & 37.6 & 48.5 \\

\rowcolor{verylightgray} & 

$\ourmethod_{\text{WU}}$                 & \textbf{41.8} & \textbf{52.6} & \textbf{58.2} & \textbf{62.3} & \textbf{39.9} & \textbf{49.3} \\

\bottomrule

\end{tabular}%

}

\end{table}

\newpage
\section{Prompts}\label{app:prompts}
\subsection{Rubric Generation}
\begin{promptenv}{Prompt for initial rubric generation.}

Role: You are a rubric designer for an LLM-as-judge system.

\par\vspace{1em}
Inputs you will receive:
\begin{itemize}
	\item Prompt: the task/question the response must answer.
    \item \textit{Responses: a set of responses to be evaluated against rubrics.}
\end{itemize}

\par\vspace{1em}
Goal: Design a comprehensive set of rubrics for evaluating responses to the given prompt. Only write rubrics you are confident about. Only propose the best new rubrics.

\par\vspace{1em}
Requirements:
\begin{itemize}
    \item Propose rubrics that collectively cover the most important dimensions needed to judge whether a response correctly and helpfully satisfies the prompt.
    \item Each rubric must be consistently judgeable across many responses (avoid vague wording like ``good'', ``nice'', ``high-quality'').
	\item Each rubric must be prompt-specific (tied to what the user asked), not generic writing advice.
	\item Each rubric should be written as a single criterion with clear, binary pass/fail boundaries. Prefer objective checks.
    \item New rubric MUST NOT answer the question directly.
    \item \textit{New rubric MUST NOT repeat any of the responses provided.}
\end{itemize}

\par\vspace{1em}
Tips for writing good rubrics:

i. MECE:
\begin{itemize}[leftmargin=1.2em]
  \item Mutually Exclusive, Collectively Exhaustive
\end{itemize}

\par\vspace{1em}
ii. Completeness:
\begin{itemize}[leftmargin=1.2em]
  \item Consider all the elements you would want to include to create a perfect response and put them into the rubric. This means including not only the facts and statements directly requested by the prompt, but also the supporting details that provide justification, reasoning, and logic for your response. Each of these elements should have a criterion because each criterion helps to develop the answer to the question from a slightly different angle.
\end{itemize}

\par\vspace{1em}
iii. No overlapping:
\begin{itemize}[leftmargin=1.2em]
  \item the same error from a model shouldn't be punished multiple times.
\end{itemize}

\par\vspace{1em}
iv. Diversity:
\begin{itemize}[leftmargin=1.2em]
  \item The rubric items should include variable types of information.
  \item If all criteria are like ``the response mentions A'', ``the response mentions B'', then this is not a good rubric.
\end{itemize}

\par\vspace{1em}
v: How many rubric items for each prompt
\begin{itemize}[leftmargin=1.2em]
  \item There is no golden standard, and the desired number of rubrics varies by accounts and task types.
  \item Write rubrics that cover all aspects of an ideal response.
\end{itemize}

\par\vspace{1em}
vi: How many rubric items to fail
\begin{itemize}[leftmargin=1.2em]
  \item A good rule of thumb is that the model fails on 50 per cent of rubrics items
\end{itemize}

\par\vspace{1em}
vii: Atomicity / Non-stacked
\begin{itemize}[leftmargin=1.2em]
  \item Each rubric criterion should evaluate exactly one distinct aspect. Avoid bundling multiple criteria into a single rubric. Most stacked criteria with the word ``and'' can be broken up into multiple pieces.
  \item BAD: Response identifies George Washington as the first U.S. president and mentions he served two terms.
  \item GOOD: Response identifies George Washington as the first U.S. president.
  \item GOOD: Response mentions that George Washington served two terms.
\end{itemize}

\par\vspace{1em}
viii: Specificity
\begin{itemize}[leftmargin=1.2em]
  \item Criteria should be binary (true or false) and objective.
  \item Avoid vague descriptions (e.g., ``the response must be accurate'' is vague).
  \item Example: ``The response should list exactly three examples.''
\end{itemize}

\par\vspace{1em}
ix: Self-contained
\begin{itemize}[leftmargin=1.2em]
  \item Each criterion should contain all the information needed to evaluate a response. For example: ``Mentions the capital city of Canada'' \(\rightarrow\) BAD; ``Mentions the capital city of Canada is Ottawa'' \(\rightarrow\) GOOD.
\end{itemize}

\par\vspace{1em}
x: Criterion should be verifiable without requiring external search.
\begin{itemize}[leftmargin=1.2em]
  \item BAD: Response names any of the Nobel Prize winners in Physics in 2023
  \item GOOD: Response names any of the following Nobel Prize winners in Physics in 2023: Pierre Agostini, Ferenc Krausz, or Anne L'Huillier.
\end{itemize}

\par\vspace{1em}
Below are the inputs:
\begin{itemize}[leftmargin=1.2em]
  \item \texttt{Here is the user prompt for which we want to generate a rubric:}\\
        \texttt{PROMPT: \{prompt\}}
  \item {\itshape
        \texttt{Here are a list of reference responses:}\\
        \texttt{RESPONSES: \{responses\}}
        }
\end{itemize}

\par\vspace{1em}
Output STRICTLY in below format. No other text is allowed:
\begin{itemize}[leftmargin=1.2em]
    \item \texttt{<RUBRIC> Rubric 1 </RUBRIC>}
    \item \texttt{<RUBRIC> Rubric 2 </RUBRIC>}
    \item \texttt{...}
\end{itemize}

\end{promptenv}

Note: \textit{Italic} parts are only included when generating rubrics with sampled responses.

\begin{promptenv}{Prompt for rubric decomposition.}

Role: You are a rubric designer for an LLM-as-judge system.

\par\vspace{1em}
Inputs you will receive:
\begin{itemize}[leftmargin=1.2em]
	\item Prompt: the task/question the response must answer.
    \item Responses: a set of responses to be evaluated against rubrics.
	\item Current rubric: criterion currently used by a judge. This rubric has already been satisfied by multiple responses, so it is too coarse and fail to distinguish response quality.
    \item Other rubrics: Other rubrics that the new rubric must NOT overlap with.
\end{itemize}

\par\vspace{1em}
Goal: Propose exactly TWO new rubrics that are more granular than the existing ones and can better differentiate candidate responses. Only write rubrics you are confident about. Only propose the best new rubrics.

\par\vspace{1em}
What ``more granular'' means (requirements):
\begin{itemize}[leftmargin=1.2em]
    \item Each new rubric must target a specific, discriminative dimension of quality that is not sufficiently captured by the existing rubrics (e.g., completeness of key sub-steps, correctness of constraints, justification quality, handling of edge cases, faithfulness to prompt format, etc.).
    \item New rubrics should NOT miss critical information contained in the existing rubric.  
	\item Each rubric must be consistently judgeable across many responses (avoid vague wording like ``good'', ``nice'', ``high-quality'').
	\item Each rubric must be prompt-specific (tied to what the user asked), not generic writing advice.
	\item Each rubric should be written as a single criterion with clear, binary pass/fail boundaries. Prefer objective checks.
    \item New rubric MUST NOT repeat any of the responses provided.
    \item New rubric MUST NOT answer the question directly.
\end{itemize}

\par\vspace{1em}
Tips for writing good rubrics:

\textit{The same as those in the prompt for initial rubric generation. These have been omitted here for brevity.}

\par\vspace{1em}
Below are the inputs:
\begin{itemize}[leftmargin=1.2em]
  \item \texttt{Here is the user prompt for which we want to generate a rubric:}]\\
  \texttt{PROMPT: \{prompt\}}
  \item \texttt{Here are a list of reference responses:}\\
  \texttt{RESPONSES: \{responses\}}
  \item \texttt{Here is existing rubric to improve:}\\
  \texttt{RUBRIC: \{rubric\}}
  \item \texttt{Here is other rubric that the new rubric should NOT overlap with:}\\
  \texttt{OTHER RUBRIC: \{other\_rubric\}}
\end{itemize}

\par\vspace{1em}
Output STRICTLY in below format. No other text is allowed:
\begin{itemize}[leftmargin=1.2em]
    \item \texttt{<RUBRIC> New rubric 1 </RUBRIC>}
    \item \texttt{<RUBRIC> New rubric 2 </RUBRIC>}
\end{itemize}

\end{promptenv}

\subsection{Filtering}

\begin{promptenv}[label={pr:overlap}]{Prompt for checking overlapping rubrics.}
You are checking whether a new rubric substantially overlaps with ANY of the existing rubrics. If ANY overlap is found, output YES; otherwise output NO.

\par\vspace{1em}
Definition of substantial overlapping:
\begin{itemize}[leftmargin=1.2em]
  \item The new rubric has the same intent as an existing rubric, or is a strict subset/superset of it, or \(\geq 70\%\) of its meaning is covered by the existing rubric so that applying both would not materially change scoring outcomes.
  \item Match on meaning, not wording. Treat synonyms, paraphrases, and inverses with the same effect as overlapping (e.g., ``be concise'' \(\approx\) ``avoid unnecessary verbosity'').
  \item Ignore trivial phrasing, tone, and example differences unless they change the requirement.
\end{itemize}

\par\vspace{1em}
Edge cases:
\begin{itemize}[leftmargin=1.2em]
  \item If scopes are disjoint (different capability/goal) \(\rightarrow\) NO.
  \item If the new rubric adds only minor qualifiers (e.g., ``clearly''/``appropriately'') without changing evaluation \(\rightarrow\) YES.
  \item If the new rubric merely narrows the context while keeping the same criterion (subset) or broadens it (superset) \(\rightarrow\) YES.
\end{itemize}

\par\vspace{1em}
Input format:
\begin{itemize}[leftmargin=1.2em]
  \item \texttt{EXISTING\_RUBRICS: \{existing\_rubrics\}}
  \item \texttt{NEW\_RUBRIC: \{new\_rubric\}}
\end{itemize}

\par\vspace{1em}
Output STRICTLY in below format. No other text is allowed:
\begin{itemize}[leftmargin=1.2em]
  \item \texttt{<EVALUATION> YES/NO </EVALUATION>}
\end{itemize}
\end{promptenv}

\begin{promptenv}[label={pr:conflict}]{Prompt for checking conflicting rubrics.}
You are checking whether a new rubric expresses opposite meaning
of ANY of the existing rubrics. If ANY opposition is found, output YES; otherwise output NO.

\par\vspace{1em}
Definition of opposition:
\begin{itemize}[leftmargin=1.2em]
  \item ``Opposite'' means the new rubric asserts the negation or reverse polarity of the same requirement, property, or direction as an existing rubric.
  \item Examples:
  \begin{itemize}[leftmargin=1.6em]
    \item require X  \(\leftrightarrow\)  forbid/avoid X
    \item must include X  \(\leftrightarrow\)  must NOT include X
    \item prefer more of X  \(\leftrightarrow\)  prefer less of X (same X, opposite direction)
    \item answer should be optimistic  \(\leftrightarrow\)  answer should be pessimistic
  \end{itemize}
  \item Do NOT flag different targets or contexts.
  \item Do NOT flag orthogonal dimensions (e.g., tone vs citations or ``be clear'' vs ``be concise'')
  \item Do NOT flag mere differences in emphasis, strength, scope, or style.
  \item Do NOT flag stricter/looser thresholds unless they clearly reverse direction on the same axis (e.g., ``maximize brevity'' vs ``maximize elaboration'' = opposite; ``\(\leq 120\) words'' vs ``\(\leq 150\) words'' = NOT opposite).
\end{itemize}

\par\vspace{1em}
Input format:
\begin{itemize}[leftmargin=1.2em]
  \item \texttt{EXISTING\_RUBRICS: \{existing\_rubrics\}}
  \item \texttt{NEW\_RUBRIC: \{new\_rubric\}}
\end{itemize}

\par\vspace{1em}
Output STRICTLY in below format. No other text is allowed:
\begin{itemize}[leftmargin=1.2em]
  \item \texttt{<EVALUATION> YES/NO </EVALUATION>}
\end{itemize}
\end{promptenv}

\subsection{Evaluation}
\begin{promptenv}[label={pr:rubricjudge}]{Judge prompt for rubric-based response evaluation.}
You are a judge, evaluating whether a response satisfies the given rubric. If the response satisfies the criterion of the rubric, output YES; otherwise output NO.  

\par\vspace{1em}
Requirement:
\begin{itemize}[leftmargin=1.2em]
      \item You must follow the rubric strictly, and only consider the criteria listed in the rubric.
      \item You must NOT consider any other factors, such as your own opinions or external knowledge.
\end{itemize}

\par\vspace{1em}
Below between \texttt{<RESPONSE>} and \texttt{</RESPONSE>} is the response to evaluate on
\begin{itemize}[leftmargin=1.2em]
  \item \texttt{<RESPONSE> \{response\} </RESPONSE>}
\end{itemize}

\par\vspace{1em}
Below between \texttt{<RUBRIC>} and \texttt{</RUBRIC>} is the rubric to evaluate on
\begin{itemize}[leftmargin=1.2em]
  \item \texttt{<RUBRIC> \{rubric\} </RUBRIC>}
\end{itemize}

\par\vspace{1em}
Output STRICTLY in below format. No other text is allowed:
\begin{itemize}[leftmargin=1.2em]
  \item \texttt{<EVALUATION> YES/NO </EVALUATION>}
\end{itemize}
\end{promptenv}

\newpage
\section{Qualitative Examples}\label{app:rubric_examples}

In this section, we present representative rubric outputs from multiple datasets, showing both the initially proposed rubrics and the corresponding refined rubrics produced through our recursive decomposition and filtering process.

Note: Due to the large number of task prompts in each dataset, and the considerable length of some prompts and their associated rubric sets, we do not display all instances, as doing so would be prohibitively difficult to read. We present representative examples to qualitatively illustrate how \ourmethod enhances both the granularity and coverage of rubric-based evaluation.

\begin{tcolorbox}[breakable,colback=orange!5,colframe=orange!40!black,title=\textbf{Example from JudgeBench Dataset}]

Prompt: 

A college student initiated a criminal case against his former roommate, claiming the former roommate snuck in and removed his state-of-the-art home theater equipment from his house. The college student took the complaint papers that were approved by the prosecution and filed them with a magistrate court. An ongoing police investigation later determined, however, that the college student was a vengeful person not averse to creating stories, and that he and the former roommate had been arguing over several related issues. They also determined that he never had any quality sound or video equipment. The police dropped the prosecution, and the criminal case was dismissed at the preliminary hearing. When the former roommate filed a civil case against the college student alleging the tort of malicious prosecution, will evidence of the student's reputation for dishonesty be admissible and why? 

\begin{itemize}
    \item (A) Yes, because the student's reputation for dishonesty directly relates to the tort of malicious prosecution.
    \item (B) No, because reputation evidence is considered too prejudicial to be admissible.
    \item (C) Yes, because dishonesty at any time in the past is always relevant to a malicious prosecution complaint.
    \item (D) Yes, because the student's dishonesty directly affects the credibility of his initial criminal complaint.
    \item (E) Yes, because reputation evidence, just like most character evidence, is virtually always admissible.
    \item (F) Yes, because even though character evidence is too collateral to be admitted for circumstantial evidence, it is admissible if it is directly at issue in the case.
    \item (G) No, because character evidence is only admissible if it is directly related to the act in question.
    \item (H) No, because the past reputation of a person is not considered as evidence in court.
    \item (I) Yes, because while most character evidence is not admissible, some is, and reputation evidence is always admissible.
    \item (J) No, because character evidence is never admissible in any case.
\end{itemize}
If you cannot determine the correct multiple-choice answer, take your best guess. Once you have your answer, please duplicate that letter five times in a single string. For example, if the answer is K, then write KKKKK. Let's think step by step.

\par\vspace{1em}
Initial rubrics:
\begin{enumerate}[leftmargin=1.2em]
    \item Correctly identifies that the admissibility of reputation evidence depends on its relevance to an element of the tort of malicious prosecution.
    \item Distinguishes between general character evidence and reputation evidence admissible when character is directly at issue.
    \item Correctly considers the legal elements required to prove malicious prosecution.
    \item Accurately addresses whether the student’s reputation for dishonesty is being used for propensity reasoning or as direct evidence of an element in the case.
    \item Avoids overgeneralization regarding the admissibility of reputation or character evidence.
    \item Acknowledges that rules of evidence may treat civil and criminal cases differently regarding character evidence.
\end{enumerate}

\par\vspace{1em}
Final rubrics:
\begin{enumerate}[leftmargin=1.2em]
    \item Clearly identifies the key elements required to prove malicious prosecution: institution of criminal proceedings, termination in plaintiff’s favor, absence of probable cause, and malice.
    \item Correctly articulates the general rule that character evidence is inadmissible to prove a person acted in conformity with that character on a particular occasion, known as the propensity rule.
    \item Accurately identifies the exception that character evidence is admissible when character is an essential element of a claim or defense, known as character being directly at issue.
    \item Explains how the student’s reputation for dishonesty directly relates to proving lack of probable cause and malice in a malicious prosecution case.
    \item Explicitly mentions the role of the student’s reputation for dishonesty in determining whether the student had an honest belief that a crime occurred.
    \item Assesses the relevance of character evidence by determining whether it significantly impacts a central factual question or element in the case, beyond mere propensity.
    \item Assesses the impact of admitting reputation evidence on the balance of prejudice versus probative value, specifically in relation to a malicious prosecution claim.
    \item Assesses the contextual relevance and potential prejudicial impact of introducing reputation evidence of dishonesty, weighing its probative value against its potential to unfairly bias the court.
    \item Explores the connection between the student’s reputation and the historical pattern of behavior, evaluating whether this pattern offers insight into a motive or strategy crucial for the malicious prosecution claim, without relying on the propensity argument.
    \item Clearly distinguishes between the legal concept of reputation and character as they apply to the admissibility of evidence, and explains any differences in their impact on the case.
\end{enumerate}
\end{tcolorbox}

\par\vspace{3em}

\begin{tcolorbox}[breakable,colback=orange!5,colframe=orange!40!black,title=\textbf{Example from PPE Preference Dataset}]
Prompt:

Write an email for selling my website speed optimization services to e-commerce business owners. Make it at short as possible. Focus in benefits (revenue or money) that i make for them.

\par\vspace{1em}
Initial rubrics:
\begin{enumerate}[leftmargin=1.2em]
    \item Mentions website speed optimization as the core service offered.
    \item Specifically targets e-commerce business owners as the intended audience.
    \item Includes at least one concrete business benefit tied to money or revenue.
    \item Maintains a short and concise format (no more than ~100 words).
    \item Avoids technical jargon or implementation details unrelated to benefits.
    \item Includes a clear call to action (e.g., reply, book a call, request audit).
    \item Avoids generic claims unrelated to website performance or monetary benefit.
\end{enumerate}

\par\vspace{1em}
Final rubrics:
\begin{enumerate}[leftmargin=1.2em]
    \item Mentions that faster loading times lead to higher sales or increased revenue.
    \item Advises that every second of delay can reduce sales or mentions specific loss metrics.
    \item Offers a specific service aimed at optimizing website speed for e-commerce.
   \item Includes a call to action inviting the recipient for further engagement, such as a chat or a video.
   \item Uses an analogy or relatable scenario to illustrate the impact of website speed on sales or customer retention.
   \item Provides an option to receive additional information or a report, implying a no-obligation offer.
   \item References an impact of speed on conversion rates, either generally or with specific percentages.
   \item Utilizes a subject line that directly mentions website speed and its impact on sales/revenue.
    \item Highlights specific benefits for customer satisfaction or user experience beyond direct revenue increase, such as improved browsing experience or reduced bounce rates.
    \item Craft a customizable template by including placeholders for recipient's name and company name.        
    \item Establishes credibility by referencing a well-known company or study related to website optimization and its impact on revenue.
    \item Utilizes persuasive language that emotionally appeals to the business owner's pain points, like frustration from slow checkout or fear of losing customers, to emphasize urgency and drive interest in services offered.
    \item Incorporates a sense of urgency or a compelling reason to act quickly within the email message.
    \item Emphasizes the competitive advantage gained by having a faster-loading website compared to other e-commerce businesses.
    \item Initiates engagement by offering a free initial assessment or audit of the recipient's current website performance to demonstrate potential improvements.
    \item Explicitly quantifies potential revenue increase with specific dollar amounts, contrasting current losses due to slow website speed with projected gains post-optimization.
    \item Incorporates industry-specific examples or case studies to demonstrate proven results of website speed optimization in increasing e-commerce sales.
    \item Clearly differentiates the email by employing a visually unique format or structure that stands out, such as bullet points, bold headings, or highlighting key revenue benefits, to enhance readability and impact.
    \item Measure the concise and efficient use of language in the email, aiming to keep the message succinct yet informative, meeting the prompt's requirement for brevity.
    \item Mentions strategies or tools used in the website speed optimization, detailing specific technologies or methodologies applied to achieve faster loading times, such as caching, image optimization, or content delivery networks.
    \item Specifies how increased website speed can amplify the business's online visibility or improve SEO rankings, thus contributing indirectly to revenue growth.
    \item Informs the recipient that an optimized website speed can lead to improved search engine rankings, thereby attracting more organic traffic and increasing potential sales.
    \item Demonstrates insight into user behavior changes stemming from improved website speed, such as increased time spent on site or higher interaction rates.
    \item Incorporates a personalization element that leverages the recipient's specific industry or e-commerce niche, showcasing tailored expertise in optimizing website speed for their unique business needs.
    \item Evaluates whether the email offers a personalized benefit or advantage uniquely tailored to the recipient's business or e-commerce niche, emphasizing customization over generic benefits.
    \item Specifies how optimized website speed can enhance mobile user experiences, which is crucial for e-commerce businesses to capture a significant portion of their audience.
    \item Presents a scalable solution that can easily integrate into the recipient's existing e-commerce infrastructure, emphasizing its convenience and low disruption to current operations.
    \item Clearly outlines the process or specific steps involved in the website speed optimization service to give the reader insight into how the transformation will be achieved.
    \item The email creatively uses a metaphor or storytelling that evokes a vivid image to emphasize the importance of website speed optimization, enhancing the overall persuasive impact.
    \item Ensures the email utilizes statistical evidence or factual data, such as industry benchmarks, to validate the benefits of speed optimization comprehensively, enhancing the persuasive power through quantifiable credibility.
\end{enumerate}
\end{tcolorbox}

\par\vspace{3em}
\begin{tcolorbox}[breakable,colback=orange!5,colframe=orange!40!black,title=\textbf{Example of a simple task from PPE Preference Dataset}]
Prompt:

There are 10 apples today. 1 was eaten yesterday. How many are left now? \textit{(Original prompt in Chinese -- translated to English for illustration)}

\par\vspace{1em}
Initial rubrics:
\begin{enumerate}[leftmargin=1.2em]
\item States that there are 10 apples remaining today, regardless of yesterday's consumption.
\item Explains that the apples consumed yesterday do not affect today's apple count of 10.
\item Acknowledges that the count of apples as of today is based on the statement ``There are 10 apples today'' without reference to past events.
\item Provides a rationale or explanation for considering the past apple consumption as irrelevant to the current count.
\item Avoids making quantitative calculations that incorrectly deduct apples from today's count based on past consumption.
\item Does not ambiguously imply there could be fewer or more than 10 apples today without any additional information provided in the prompt.
\end{enumerate}

\par\vspace{1em}
Final rubrics:
\begin{enumerate}[leftmargin=1.2em]
\item (Same as the initial rubrics -- omit for brevity.)
\item Provides a creative interpretation of the scenario, incorporating narrative or imaginative elements that explain the unchanged apple count.
\item Uses narrative or imaginative storytelling elements without deviating from the stated fact that today's count is 10 apples, enhancing engagement while maintaining accuracy.
\end{enumerate}
\end{tcolorbox}

\par\vspace{3em}
\begin{tcolorbox}[breakable,colback=orange!5,colframe=orange!40!black,title=\textbf{Example of WildChat Dataset}]

\textbf{Example 1}
\par\vspace{1em}

Prompt: 

What influence did Greek thought have on the emergence and development of science?

\par\vspace{1em}
Initial rubrics:
\begin{enumerate}[leftmargin=1.2em]
\item Identifies Greek thought as a historical antecedent rather than a modern scientific contributor.
\item Describes at least one conceptual contribution of Greek thought to the development of scientific reasoning.
\item Mentions at least one historical figure or school of Greek thought in a scientific or proto-scientific context.
\item Acknowledges the role of Greek thought in shaping methodological or epistemological foundations of science.
\item Avoids attributing anachronistic concepts of modern science (e.g., peer review, experimental control) directly to the Greeks.
\item Maintains focus on Greek influence specifically, not general ancient civilizations.
\item Frames the influence of Greek thought in terms of its impact on later scientific development (e.g., during the Islamic Golden Age, Renaissance, Enlightenment).
\end{enumerate}

\par\vspace{1em}
Final rubrics:
\begin{enumerate}[leftmargin=1.2em]
\item The response mentions Thales of Miletus predicting an eclipse as an example of early scientific reasoning by the Greeks.
\item The response includes an explanation of the Greek concept of 'Kosmos' as an ordered, harmonious system significant to scientific thought.
\item The response describes the influence of Pythagorean mathematics on the development of scientific principles.
\item The response discusses Socratic questioning or the Socratic method as essential for scientific inquiry.
\item The response acknowledges errors in Greek science but emphasizes their foundational contributions to scientific development.
\item The response cites Euclid's "Elements" and its impact on geometry as a part of scientific advancement.
\item The response mentions the Greek transition from mythological explanations to logical reasoning in understanding the universe.
\item The response includes reference to Greek thinkers' belief in a rational, intelligible universe governed by laws.
\item The response discusses Aristotle’s role in the systematization and classification of the natural world.
\item The response explores Greek contributions to observational and empirical methods, such as Hippocrates' emphasis on documenting symptoms, in the development of scientific methodology.
\item The response addresses the synthesis of empirical observation and logical reasoning as a Greek methodological advancement in the pursuit of understanding the natural world.
\item The response identifies the impact of Greek dialectic methods on shaping the peer review process used in modern science.
\item The response explains the role of Greek mathematics, particularly its use in describing physical phenomena, in influencing later scientific approaches to the natural world.
\item The response illustrates how Greek thought contributed to the formation of a foundational ethical framework for science, highlighting examples like the Hippocratic Oath.
\item The response explores the role of Archimedes in advancing principles of physics and engineering, illustrating the practical application of Greek thought to material reality.
\item The response illustrates how Greek philosophical discourse in public spaces, such as the Agora, contributed to the communal and social nature of scientific inquiry.
\item The response explores the role of Greek philosophical schools, such as the Atomists or Stoics, in shaping fundamental principles that influenced later scientific exploration.
\item The response discusses the role of Greek thought in shaping the dual approach of both theoretical abstraction and practical experimentation which became fundamental to later scientific advancements, despite the Greeks not fully implementing experimentation themselves.
\item The response showcases how Greek thought laid the groundwork for interdisciplinary thinking, mixing philosophy, mathematics, and early natural sciences to broaden approaches in scientific exploration.
\item The response highlights the symbolic aspect of Greek thought as initiating a "cathedral of knowledge," metaphorically illustrating the foundational impact on science across millennia.
\item The response highlights the impact of Greek thought in establishing a systematic methodology for distinguishing between subjective beliefs and objective facts, facilitating a structured framework for scientific inquiry.
\item The response explores the lasting cultural and intellectual impact of Greek thought on the Renaissance and subsequent scientific revolution, highlighting the transfer and transformation of Greek ideas across different historical periods.
\item The response explains how Greek thought established an early conceptual framework for distinguishing between qualitative and quantitative analysis in scientific inquiry.
\item The response explores how Greek advancements in abstract thought, particularly through Plato's ideal forms, contributed to the long-term development of scientific concepts by encouraging the pursuit of universal principles and truths.
\item The response highlights the contribution of Greek exploration of metaphysics and ontology—concepts of being and existence—and their impact on shaping scientific inquiry into the nature of reality.
\item The response outlines the Greek philosophical commitment to deductive reasoning as a critical factor in establishing structured scientific theories.
\end{enumerate}

\par\vspace{3em}
\textbf{Example 2}
\par\vspace{1em}
Prompt:

What skill can I learn in six months and generate a lot of money?

\par\vspace{1em}
Initial rubrics:
\begin{enumerate}[leftmargin=1.2em]
\item Proposes a skill that is realistically learnable within approximately six months.
\item Proposes a skill that has clear, established paths to monetization.
\item Clearly identifies the skill being recommended.
\item Avoids unrealistic earnings claims (e.g., “become a millionaire instantly”) or unverifiable hype.
\item Connects the proposed skill to a plausible use case or real-world application.
\item Maintains relevance to both time constraint (six months) and financial goal (generating a lot of money).
\end{enumerate}

\par\vspace{1em}
Final rubrics:
\begin{enumerate}[leftmargin=1.2em]
\item Identifies and recommends a specific skill that can be learned in six months and is capable of generating substantial income.
\item Provides justification for why the skill chosen is lucrative, referring to current demand and market trends.
\item Lists actionable steps or a clear plan on how to learn the skill within the six-month time frame.
\item Includes examples of tools or platforms necessary to master the skill.
\item Discusses potential income or earning opportunities related to the skill.
\item Addresses the need for the skill to solve specific business problems or provide substantial value to clients.
\item Suggests ways to monetize the skill, such as freelance opportunities, consulting, or creating digital products.
\item Emphasizes the importance of differentiating oneself with complementary skills or additional strategies beyond the primary skill.
\item Mentions the importance of building a portfolio or case studies to showcase expertise in the chosen skill.
\item Encourages pursuing online courses, certification, or self-directed learning as part of the skill acquisition process.
\item Evaluates the long-term sustainability and adaptability of the skill in relation to evolving market conditions and potential technological advancements.
\item Evaluates the transformation of the learned skill into a clearly defined business model, detailing how the skill is integrated into a sustainable, repeatable, and scalable business process.
\item Evaluates the practicality of transitioning from learning the skill to implementing it in real-world scenarios within the six-month time frame.
\item Evaluates the potential scalability or expansion of the skill into a broader business or service offering in the future, beyond just immediate income generation.
\item Evaluates the effectiveness of the skill in addressing industry-specific challenges, detailing how it creates a competitive advantage or fills a gap in the current market.
\item Evaluates the creativity and uniqueness of the proposed learning path by assessing whether the response provides a vision or framework for integrating the skill into the user’s broader personal or professional development roadmap.
\item Evaluates potential challenges or pitfalls in learning the chosen skill and suggests methods to overcome them.
\item Evaluates the clarity and realism of the proposed six-month learning timeline in the context of the user's current knowledge and experience level, assessing whether the response tailors the learning schedule to various starting points of potential learners.
\item Assesses whether the response includes a strategy for gaining the first practical experience or opportunities (e.g., internships, volunteer projects, or initial clients) in the chosen skill to build credibility and real-world expertise.
\item Examines the feasibility and validity of the skill learning timeline by analyzing the skill's complexity, required foundational knowledge, and intensity of the proposed learning schedule.
\item Assesses the completeness and coherence of the narrative or framework describing how the proposed skill will lead to financial success, ensuring it logically connects individual learning steps to the eventual income objectives.
\item Assesses the ability of the response to convey the transformative impact of the skill on both personal growth and career trajectory, highlighting the broader implications of mastering the skill beyond immediate financial gain.
\item Evaluates the response’s inclusion of strategies to enhance personal branding related to the skill, including creating an online presence or sharing work publicly to attract potential clients or employers.
\item Assesses whether the response effectively illustrates a unique marketing or networking strategy to connect with potential clients or build a professional reputation in the first stages of the skill application.
\item Emphasizes the need to practice and iterate on real-world projects or scenarios relevant to the chosen skill, demonstrating iterative improvement and application of feedback.
\item Assesses the response for providing specific examples of successful individuals or case studies who have effectively learned and monetized the skill within a similar timeframe, showcasing real-world applicability and success stories.
\item Assesses the depth of detail and specificity in outlining how the proposed skill integrates with existing competencies or knowledge, illustrating how this integration enhances the potential for success and financial gain.
\item Assesses the degree to which the response incorporates practical examples or case studies demonstrating the real-world application and impact of the skill learned, highlighting any transformations in business operations or financial performance.
\item Assesses the alignment of the skill with the user's current abilities and interests, ensuring that the skill is not only compatible with but also builds upon their existing strengths, knowledge, or passions to maximize successful implementation and retention.
\item Assesses the ability to effectively communicate and articulate the value proposition of the learned skill to prospective clients or employers, ensuring that the pitch is tailored to their specific needs and objectives.
\item Assesses the response's ability to identify opportunities for networking or collaboration within the industry relevant to the chosen skill, and how this could accelerate learning and increase potential income.
\end{enumerate}

\par\vspace{3em}
\textbf{Example 3}
\par\vspace{1em}
Prompt:

Write me an essay about monkeys.

\par\vspace{1em}
Initial rubrics:
\begin{enumerate}[leftmargin=1.2em]
\item The response is structured in essay format with an introduction, body, and conclusion.
\item The essay remains focused on the topic of monkeys throughout the response.
\item The response refers specifically to monkeys and does not conflate them with apes or other primates.
\item Includes at least one factual detail about monkeys (e.g., diet, habitat, species, behavior).
\item Avoids listing-only or bullet-point format; presented as continuous prose.
\item Does not consist entirely of generic filler or vague commentary.
\item Maintains a neutral or informative tone appropriate for an essay.
\end{enumerate}

\par\vspace{1em}
Final rubrics:
\begin{enumerate}[leftmargin=1.2em]
\item Includes an exploration of monkeys' physical abilities, such as their agility and acrobatics.
\item Mentions the social structures and behaviors observed in monkey societies.
\item Describes the intelligence and problem-solving abilities of monkeys.
\item Examines the evolutionary connection between monkeys and humans.
\item Discusses the ecological role of monkeys, such as seed dispersers.
\item Provides factual information about the diversity of monkey species.
\item Reflects on the emotional experiences shared between humans and monkeys.
\item Utilizes symbolic language to enhance understanding of monkeys beyond biological terms.
\item Considers the environmental habitat diversity in which different monkey species thrive, including their geographical distribution across continents.
\item Compares and contrasts different monkey species to highlight their unique characteristics and ecological niches.
\item Highlights the role of monkeys in scientific research and their contributions to advancements in our understanding of human cognition and disease.
\item Explains the impact of human activity, such as the illegal wildlife trade, on monkey populations.
\item Considers the historical depiction of monkeys in literature and folklore across different civilizations.
\item Discusses the use of metaphorical elements to draw parallels between monkey behavior and broader existential themes or philosophical questions.
\item Discusses the sensory experiences and communication methods of monkeys, such as vocalizations, facial expressions, and tactile behaviors, and their significance in their daily life and interactions.
\item Analyzes the role of monkeys in fostering cross-cultural connections and exchanges, highlighting examples where different cultures have interacted with or been influenced by monkey-related myths, symbols, or depictions.
\item Evaluates the symbolic representation of monkeys as tricksters or wise figures in folklore and what this reveals about human self-perception and cultural values.
\item Analyzes the role of monkeys in maintaining ecological balance and biodiversity through their interactions with other species.
\item Critically analyzes the ethical considerations and responsibilities associated with human interactions and interventions with monkey species, including the implications of research and conservation efforts.
\item Analyzes the role of monkeys in facilitating biodiversity through their interactions with other species in their ecosystems.
\item Discusses the ethical considerations in scientific research involving monkeys, including the balance between scientific advancement and animal welfare.
\item Analyzes the ethical considerations and moral obligations involved in the conservation of monkey habitats and populations.
\item Presents a nuanced exploration of monkeys' interactions with their environment, focusing on their role as adaptors within ecosystems and their ability to thrive across varying environmental conditions.
\item Assesses the depiction of unique communication methods and the role they play in both cooperative interactions and competitive scenarios within monkey societies.
\item Assesses the aesthetic portrayal of monkeys in the essay, considering the use of descriptive imagery to evoke sensory experiences and emotional responses in the reader.
\item Examines the role of monkeys in inspiring artistic and creative expressions, such as in visual arts, music, or dance, and how these reflect human interpretations of monkey characteristics and significance.
\item Explores the challenges related to the social dynamics within monkey populations when faced with environmental pressures or changes, such as altered food availability or increased competition for resources.
\item Discusses the role of monkeys in their local food webs and their interactions with other species as both prey and predators.
\item Assesses the anthropomorphism of monkey behaviors and characteristics, examining how human-like traits are attributed to monkeys in the narrative.
\item Evaluates the physiological adaptations of monkeys that enable them to thrive in varied environmental conditions, such as skin type variations and dietary specializations.
\item Analyzes the adaptation mechanisms monkeys have developed to thrive in varying environmental pressures and challenges, such as climate variations and human-induced habitat changes.
\item Discusses the role of monkeys in influencing art forms, such as visual art, music, and dance, and how these representations impact cultural perception and creative expression.
\item Explores the role of monkeys in allegorical storytelling, examining how they serve as vehicles for moral lessons and social commentary in diverse cultural narratives.
\end{enumerate}
\end{tcolorbox}

\end{document}